\documentclass{article}

\usepackage[preprint]{neurips_2026}

\usepackage[utf8]{inputenc}
\usepackage[T1]{fontenc}
\usepackage{hyperref}
\usepackage{url}
\usepackage{booktabs}
\usepackage{amsfonts}
\usepackage{nicefrac}
\usepackage{microtype}
\usepackage{xcolor}

\usepackage{enumitem}

\usepackage{amsmath}
\usepackage{amssymb}
\usepackage{amsthm}

\newtheorem{theorem}{Theorem}

\newtheorem{prop}{Proposition}
\newtheorem{assumption}{Assumption}
\theoremstyle{definition}
\newtheorem{definition}{Definition}
\newtheorem{remark}{Remark}

\usepackage{graphicx}
\usepackage{subcaption}
\usepackage{float}
\usepackage{tikz}
\usetikzlibrary{positioning,calc}
\usepackage{pgfplots}
\pgfplotsset{compat=1.18}
\usepackage{comment}
\usepackage[ruled]{algorithm2e}

\SetAlFnt{\small}
\SetAlCapFnt{\small}
\SetAlCapNameFnt{\small}
\SetAlCapHSkip{0pt}
\IncMargin{-\parindent}

\title{Causal Multi-Task Demand Learning}

\author{
  Varun Gupta\\
  Dept.\ of Operations and Information Systems\\
  University of Utah\\
  \texttt{varun.gupta@eccles.utah.edu}
  \And
  Vijay Kamble\\
  Dept.\ of Information and Decision Sciences\\
  University of Illinois Chicago\\
  \texttt{kamble@uic.edu}
}

\begin{document}

\maketitle

\begin{abstract}
We study multi-task demand learning for retail pricing, where firms seek to estimate heterogeneous linear price-response functions across many decision contexts. Each context has rich covariates but limited price variation, motivating transfer across tasks. However, such transfer is complicated by endogeneity: prices may be arbitrarily correlated with unobserved task-level demand determinants.

We propose a meta-learning framework that identifies the conditional mean of task-specific causal demand parameters despite this confounding, provided each task contains at least two distinct, locally exogenous price points. The key idea is to carefully design the information available to the learner. We condition on the full price history to address cross-task confounding, but withhold the demand outcomes associated with two locally exogenous price points to preserve identifiability. These withheld outcomes provide supervision through a randomized query.

We establish that this information design is maximally uniformly valid for the squared-loss objective: any nonconstant refinement of the conditioning set that reveals information about the withheld outcomes or the query index admits a data-generating process under which the objective no longer identifies the conditional mean causal target. Synthetic experiments demonstrate improved estimation of causal demand parameters under strong confounding, while retail experiments show improved held-out demand prediction relative to standard transfer-learning baselines.
\end{abstract}

\section{Introduction}
Large retail firms routinely set prices across a wide range of selling contexts. These contexts may correspond to a single product offered across geographically dispersed stores and channels within an omnichannel network (context is the geography and store-level features) or to distinct products within an e-commerce catalog (context is the product features). In all cases, each decision environment is shaped by heterogeneous demand drivers, including customer demographics, competitive conditions, local preferences, and product-specific attributes. Consequently, effective pricing requires accurate estimation of context-specific price-response functions.

Estimating such heterogeneous demand responses presents a statistical challenge. Within any given store or product, price variation is typically sparse: operational frictions often limit exposure to only a small set of prices over extended periods. In contrast, the cross-sectional dimension is large, with many stores or products and rich covariate information. This imbalance naturally motivates multi-task learning or partial pooling approaches \citep{caruana1997multitask, baxter2000model, gelman2013bayesian}, in which a shared model maps observable characteristics to task-specific demand parameters, borrowing strength across tasks to compensate for limited within-task variation.

However, many factors that determine prices are unobserved or only partially observed by the econometrician. As a result, prices may remain systematically correlated with the unobserved demand determinants across tasks even after conditioning on the observed covariates, leading to {\it unobserved confounding}. This raises a central question: can cross-task transfer learning improve estimation of causal demand parameters despite such confounding, and under what conditions? 

We formalize a model of multi-task pricing with heterogeneous linear price responses, where observed prices may depend on latent demand determinants. Standard pooled and meta-learning estimators generally converge to biased policy-dependent estimands rather than causal targets in this model. We make the following main contributions:
\begin{enumerate}[leftmargin=*]
\item We propose a new information design principle for meta-learning in this setting, termed Decision-Conditioned Masked-Outcome Meta-Learning (DCMOML). Assuming at least two distinct locally exogenous price points per task, we show that \textsc{DCMOML} identifies and consistently estimates the conditional mean of task-specific causal demand parameters given a specified information set. This information set includes all task prices to address confounding, while masking two demand observations corresponding to locally exogenous prices to address the identifiability issues that arise when prices are fully observed by the meta-learner. Crucially, supervision is provided by query randomization over the two obfuscated demand outcomes.
\item We show that the DCMOML information set is maximally uniformly valid for its two-point squared-loss supervision: for every nonconstant refinement that reveals withheld-outcome or query-index information, some admissible data-generating process makes the refined population loss fail to recover the conditional mean target uniquely.
\item On synthetic data with known parameters, DCMOML approaches the Bayes floor of its information set and outperforms pooled regression, empirical Bayes, and standard support--query meta-learning under strong confounding; standard meta-learning can be more accurate for intercepts when confounding is weak. On real retail data, DCMOML improves prediction at withheld prices.
\end{enumerate}
We note that task-level identification is not the primary difficulty in our setting. Under our model, two distinct locally exogenous price points within a task already supply two valid conditional moment equations for the causal parameter, and hence a two-point OLS estimate of it. However, such estimates are highly variable when each task contains only a small number of observations, and may lack finite moments when the two prices are close. Our goal is to improve precision by borrowing strength across tasks through transfer learning. And this is precisely where confounding becomes consequential. 

If each task contained only a single price point, the problem would collapse to a standard regression with endogenous prices, where unbiased estimation is generally impossible without additional assumptions or valid instruments. Our main insight is that locally exogenous price variation within each task allows us to use cross-task learning to improve precision, without requiring prices to be exogenous across tasks.
 
Overall, our contribution bridges empirical demand estimation under endogenous pricing with cross-task learning in data-sparse environments. Improved estimation of these context-specific causal demand primitives can enable more principled downstream pricing decisions in retail operations, hence our contribution is of practical value.

\section{Related Work}
\label{sec:literature}
{\bf Demand estimation under endogenous prices.}
A central concern in empirical demand estimation is that prices are often chosen in response to
demand information unobserved by the econometrician. Classical approaches address this
endogeneity using instruments, control functions, panel variation, or structural equilibrium
restrictions \citep{berry1994estimating,berry1995automobile,hausman1996valuation,
nevo2001measuring,petrin2010control,berry2014identification,angrist2009mostly,
wooldridge2010econometric}. DeepIV extends the instrumental-variables approach to flexible
nonlinear prediction by using instruments to generate conditionally exogenous treatment variation
\citep{hartford2017deep}. We assume neither excluded instruments nor conditional independence of
prices from latent demand determinants. Instead, DCMOML leverages repeated, locally exogenous decisions
within each task to identify a conditional mean causal target in a many-task, few-observation regime.

{\bf Partial pooling, empirical Bayes, and meta-learning.}
Hierarchical Bayes, random-effects models, and empirical Bayes shrinkage reduce the variance of
task-level estimation when each unit has few observations
\citep{robbins1956empirical,laird1982random,gelman2013bayesian,efron2010large}. These approaches typically rely on valid causal estimates at the task level, which our model provides through the two locally exogenous prices per task. In standard formulations with a price-independent task-parameter prior, EB/random-effects methods use realized prices only as regressors in the outcome likelihood. DCMOML additionally treats the price history as information about latent demand parameters and learns this relationship flexibly, without specifying a model of how prices are set. Multi-task and meta-learning methods are natural for demand learning with many
stores or products and sparse within-task price variation
\citep{caruana1997multitask,baxter2000model,evgeniou2004regularized,finn2017maml,
hospedales2021metalearning}. However, when prices encode latent demand information, a pooled
predictor or standard support--query meta-learner may fail to learn the causal price response. DCMOML is designed to precisely address this issue.

{\bf Causal machine learning.}
Causal trees and forests, double/debiased machine learning (DML), and R-learners estimate treatment effects using flexible regression and, in some formulations, orthogonalization \citep{athey2016recursive,wager2018estimation,chernozhukov2018double,nie2021quasi,imbens2015causal}. Their observed-confounder formulations rely on assumptions such as unconfoundedness and overlap. DML also accommodates unobserved confounding when valid instruments are available. Our setting assumes neither that observed covariates capture all confounding nor that excluded instruments are available; identification instead relies on the linear demand structure and repeated, locally exogenous observations within tasks.

{\bf Correlated random effects and random-coefficient panels.}
Our setting is closely related to correlated random-effects and random-coefficient panel models, in
which unit-level heterogeneity may be correlated with the within-unit regressor history
\citep{swamy1970efficient,mundlak1978pooling,chamberlain1982multivariate,
chamberlain1984panel,wooldridge2010econometric}. Mundlak--Chamberlain approaches relax the
standard random-effects independence assumption by modeling the conditional mean of latent unit effects given the regressor history, often through linear covariate summaries. In these models, outcomes enter the estimating equations, while the conditioning object is typically fixed in advance as a function of the observed regressors. DCMOML adopts this conditional-random-coefficient perspective but differs in what conditioning
information is targeted and how the target is learned. In our meta-learning setting, information about the latent demand type comes
from two sources: the endogenous price history, which reflects prior pricing decisions, and the
within-task relationship between realized prices and demands, which is informative about the task's
random coefficient. This necessitates that (a) the demand outcomes enter the conditioning information, and (b) the targeted conditional mean be a flexible function of prices and outcomes as opposed to linear price summaries. These aspects create a unique identification challenge that we address using outcome masking and query randomization.

\section{Model}\label{sec:model}

{\bf Tasks, Data, and Notation.}
We study a multi-task demand learning setting with $N$ tasks indexed by $i = 1,\dots,N$. Each task can be interpreted as a distinct store selling a fixed product. Tasks are heterogeneous and described by observable covariates $Z_i \in \mathbb{R}^d$, capturing demographics, competition, geography, and other store-level features. We assume $\{Z_i\}_{i=1}^N$ are i.i.d.\ draws from an unknown distribution $\mathcal{P}_Z$. For each task $i$, we observe a sequence of $K$ price--demand pairs
\[
\{(p_{ik}, D_{ik})\}_{k=1}^K,
\]
where $p_{ik} \in \mathbb{R}$ is the posted price and $D_{ik} \in \mathbb{R}$ is realized demand. We focus on regimes with limited within-task variation, where $K$ is small (e.g., $K=2$ or $3$).

{\bf Structural Demand Model.}
Demand is assumed linear in price within each task:
\begin{equation}
D_{ik} = \theta_i^0 + \theta_i^1 p_{ik} + \epsilon_{ik},
\label{eq:linear_demand}
\end{equation}
where $\Theta_i = (\theta_i^0, \theta_i^1)^\top \in \mathbb{R}^2$ are task-specific demand parameters and $\epsilon_{ik}$ is an idiosyncratic shock. Let $P_{ik} = (1,\, p_{ik})^\top$, so that
\[
D_{ik} = P_{ik}^\top \Theta_i + \epsilon_{ik}.
\]
We assume the shocks have mean zero and finite variance, and that $\epsilon_{i1},\dots,\epsilon_{iK}$ and $(Z_i,\Theta_i)$ are mutually independent.

{\bf Heterogeneity and Shared Structure.}
We decompose task-specific parameters as
\begin{equation}
\Theta_i = g(Z_i) + \Omega_i,
\label{eq:theta_decomposition}
\end{equation}
where $g : \mathbb{R}^d \to \mathbb{R}^2$ is an unknown function capturing systematic variation across tasks, and $\Omega_i \in \mathbb{R}^2$ is an idiosyncratic task-specific deviation. We assume $\mathbb{E}[\Omega_i \mid Z_i] = 0$ and that $\{\Omega_i\}$ are i.i.d.\ with finite second moments. Substituting into \eqref{eq:linear_demand} yields
\begin{equation}
D_{ik}
=
P_{ik}^\top g(Z_i)
+
P_{ik}^\top \Omega_i
+
\epsilon_{ik}.
\label{eq:full_model}
\end{equation}
Task tuples $(Z_i,\Omega_i,p_{i1:K},\epsilon_{i1:K})$ are i.i.d.\ across tasks.

{\bf Price Assignment and Confounding.}
We allow the task-specific prices $\{p_{ik}\}_{k=1}^K$ to depend on the unobserved heterogeneity $\Omega_i$, even after conditioning on the observed context $Z_i$. Thus, prices may be endogenous through their dependence on $\Omega_i$ across tasks. We refer to this feature as {\it cross-task confounding}.

At the same time, we assume that there are at least two distinct {\it locally exogenous} prices in each task, which supply two valid conditional moment equations for $\Theta_i$ at the task level:

\begin{assumption}[Local exogeneity]
\label{ass:local_exogeneity}
For each task $i$, there exist two indices $a_i$ and $b_i$, measurable with respect to
$(Z_i,p_{i1:K})$, such that $p_{ia_i} \neq p_{ib_i}$ almost surely. Moreover, for each $k\in\{a_i,b_i\}$,
\begin{align}
\mathbb{E}\!\left[\epsilon_{ik}\mid Z_i, p_{i1:K}, \mathbf{D}_i^{-k}\right]=0,\label{ass:zero_mean}
\end{align}
where $\mathbf{D}_i^{-k}=\{D_{ij}:j\neq k\}$ denotes the task history with the $k$th demand
observation omitted.
\end{assumption}
For example, Assumption~\ref{ass:local_exogeneity} holds when prices are pre-committed, i.e.,  $p_{i1:K}$ is set before the demand shocks are realized, possibly depending arbitrarily on $(Z_i,\Omega_i)$ and on randomization independent of the shocks. Prices are then endogenous through latent demand type, yet each mean-zero shock $\epsilon_{ik}$ is independent of $(Z_i,\Omega_i,p_{i1:K},\epsilon_{i,-k})$, and hence of $(Z_i,p_{i1:K},\mathbf D_i^{-k})$. Hence the assumption holds at every index, and any two indices with distinct prices can serve as $(a_i,b_i)$.

Local exogeneity can also hold after an adaptive history: if each task ends at a fixed pair of distinct probe prices chosen from past demand before either probe outcome is observed, as in Kiefer--Wolfowitz experimentation \citep{kiefer1952stochastic}, that final pair is locally exogenous under the model's shock-independence assumptions.

\begin{assumption}[Moment condition]
\label{ass:moments}
The task parameters and the noiseless demands at the observed prices have finite second moments:
\begin{equation}
\mathbb{E}\|\Theta_i\|_2^2<\infty,
\qquad
\mathbb{E}\big[(P_{ik}^\top\Theta_i)^2\big]<\infty
\quad\text{for } k=1,\dots,K.
\label{eq:moments}
\end{equation}
\end{assumption}
These conditions are mild; for example, they hold whenever the task parameters and prices are almost surely bounded.

{\bf Causal Target and Objective.} 
The causal object of interest is the task-specific parameter $\Theta_i$, which governs demand responses under counterfactual price interventions. Under Assumption~\ref{ass:local_exogeneity}, the two locally exogenous observations supply valid conditional moment equations for $\Theta_i$, and determine a two-point OLS estimate of it. With a small $K$, however, that estimate is noisy and unstable, and its unconditional moments need not exist when the two prices are close. To obtain a less noisy estimate, in the spirit of meta-learning, our objective is to leverage the cross-task distributional information available in the many-tasks regime to target the conditional expectation of the causal parameter $\Theta_i$ given task-level information sets of observables, where the key challenge is confounding. 

\section{Baseline Transfer-Learning Approaches and Challenges}
We first discuss how standard pooling and meta-learning approaches are biased in our setting.

{\bf Shared-Model Learning.}
The simplest transfer learning approach is to ignore task-level heterogeneity and estimate only the shared mapping $g(\cdot)$ by pooling all observations. A natural baseline in this spirit is to posit a parameterized function class $\{g_\Lambda(\cdot)\}_{\Lambda \in \mathcal{H}}$ (e.g., deep neural
networks), and to estimate $\Lambda$ by pooling all observations across tasks:
\begin{equation}
\widehat{\Lambda}
\in
\arg\min_{\Lambda \in \mathcal{H}}
\sum_{i=1}^{N}\sum_{k=1}^K
\left(D_{ik} - P_{ik}^\top g_{\Lambda}(Z_i)\right)^2.
\label{eq:global-naive}
\end{equation}
This is appealing in the many-tasks regime, as it avoids noisy within-task estimation, but it is generally inconsistent under our model. Substituting the data-generating process gives
\[
D_{ik} - P_{ik}^\top g(Z_i)
=
P_{ik}^\top \Omega_i + \epsilon_{ik}.
\]
Because prices are endogenously assigned, $P_{ik}$ may be correlated with $\Omega_i$, implying
$\mathbb{E}[P_{ik}^\top \Omega_i \mid Z_i] \neq 0$,
violating the orthogonality condition required for consistency of \eqref{eq:global-naive}: a classic endogeneity failure. Appendix~\ref{ex:fail} gives an example in which confounded near-optimal pricing reverses the sign of the pooled slope. 

{\bf Meta-Learning.}
Standard support--query meta-learning attempts to leverage the benefits of pooling while allowing for per-task adaptation \citep{hospedales2021metalearning, finn2017maml}. Let $g_\Lambda : (Z_i, \mathcal{S}_i) \mapsto \widehat{\Theta}_i$, where $\mathcal{S}_i = \{(P_{i1}, D_{i1}), \dots, (P_{i,K-1}, D_{i,K-1})\}$ is a \textit{support set}. The final observation $(P_{iK}, D_{iK})$ is treated as a \textit{query}, and $\Lambda$ is estimated via
\begin{align}
\widehat{\Lambda}
\in
\arg\min_{\Lambda}
\sum_{i=1}^{N}
\left(
D_{iK} - P_{iK}^\top g_{\Lambda}(Z_i, \mathcal{S}_i)
\right)^2.
\label{eq:meta-erm}
\end{align}
While one might expect this to recover $m_i\triangleq\mathbb{E}[\Theta_i \mid Z_i, \mathcal{S}_i]$, this is not the case in general, since endogeneity leads to inconsistent estimation. For instance, if the query index is locally exogenous, the expected residual of $m_i$ is
\[
\mathbb{E}\!\left[D_{iK}-P_{iK}^\top m_i\;\middle|\;Z_i,\mathcal{S}_i\right]
=\operatorname{tr}\mathrm{Cov}(P_{iK},\Omega_i\mid Z_i,\mathcal{S}_i),
\]
which need not be zero, because $p_{iK}$ can remain correlated with $\Omega_i$ even after conditioning on $(Z_i,\mathcal{S}_i)$. And whenever it is nonzero, $m_i$ is not a minimizer of the objective in \eqref{eq:meta-erm}. Thus, meta-learning generally converges to a policy-dependent estimand. In Appendix~\ref{ex:fail}, we present an example of this failure in the same setting where shared-model learning fails. 

\section{Causal Identification}

We formulate causal identification of a conditional target as a meta-learning problem in which both the support information and query supervision must be carefully designed to address confounding. We develop this construction step by step before presenting our estimator.

{\bf Conditioning on the Full Price History.}\label{sec:full_history_not_sufficient}
Suppose the query index $K$ is locally exogenous. Because prices are confounded with latent task parameters, we first enlarge the standard support--query conditioning set of meta-learning to include the full price history. Specifically, consider the target
\begin{equation}
\mathbb{E}\!\left[
\Theta_i
\;\middle|\;
Z_i,\;
p_{i1},\dots,p_{iK},\;
D_{i1},\dots,D_{i,K-1}
\right],
\label{eq:target-full}
\end{equation}
which conditions on all endogenous price realizations, including the query price $p_{iK}$, leaving the demand $D_{iK}$ as supervision target for learning. A natural meta-learning estimator solves
\begin{equation}
\widehat{\Lambda}
\in
\arg\min_{\Lambda}
\sum_{i=1}^N
\left(
D_{iK}
-
P_{iK}^\top
g_{\Lambda}\!\big(
Z_i,\; p_{i1:K},\; D_{i,1:K-1}
\big)
\right)^2,
\label{eq:meta-ideal}
\end{equation}
with $P_{iK}=(1,p_{iK})^\top$. Because the full confounded price vector $p_{i1:K}$ is included as a covariate, this removes confounding at the level of conditional means, rendering the residual in \eqref{eq:meta-ideal} mean-zero under the target \eqref{eq:target-full}. However, it introduces a new failure mode of \textit{identifiability}. The issue arises because the meta-learner has explicit access to the query price point $p_{iK}$, and the supervision
signal for task $i$ enters only through the scalar inner product
$P_{iK}^\top \Theta_i$. Consequently, the empirical risk in \eqref{eq:meta-ideal} is
invariant to shifts of the predicted parameter vector along directions orthogonal to
$P_{iK}$. Formally, for any candidate prediction $\widehat{\Theta}_i$ and any measurable scalar
function $\phi$ of the inputs, define
\[
\widetilde{\Theta}_i
\;\triangleq\;
\widehat{\Theta}_i
+
\phi(\text{inputs}_i)
 [\,p_{iK},-1\,]^\top.
\]
Since $P_{iK}^\top [\,p_{iK},-1\,]^\top = 0$, both $\widehat{\Theta}_i$ and
$\widetilde{\Theta}_i$ achieve exactly the same objective value in
\eqref{eq:meta-ideal}. When $g_\Lambda$ is expressive, many such solutions can exist, so this objective cannot uniquely select the target in \eqref{eq:target-full}.

{\bf Identification via Partial Obfuscation of the Decision History.}
We now construct an information set and supervision policy for a meta-learner that continues to condition on all prices to address confounding, while restoring identification by preventing the learner from identifying the query regressor.

\emph{Step 1: Randomizing the query index.}
In standard meta-learning, the query index (e.g., $k=K$) is fixed, so the learner knows which price enters the loss. To avoid this, we randomize the query within a locally exogenous two-point subset. Specifically, select a locally exogenous index pair $A_i = \{a_i, b_i\} \subseteq [K]$ (which exists by Assumption~\ref{ass:local_exogeneity}), potentially as a function of $(Z_i, p_{i1:K})$. Draw $\kappa_i \sim \mathrm{Unif}\{a_i, b_i\}$ independently across tasks, and treat $(p_{i\kappa_i}, D_{i\kappa_i})$ as the query.

\emph{Step 2: Outcome masking.}
Randomization alone is insufficient: the learner can still infer $\kappa_i$ because the observed demand vector is aligned with prices, leaving exactly one missing outcome. To prevent this, we mask both candidate query outcomes and instead target
\begin{equation}
\mathbb{E}\!\left[\Theta_i \mid Z_i,\; p_{i1:K},\; \mathbf{D}_i^{-A_i}\right],
\label{eq:target-kminus2_twopoint}
\end{equation}
where $\mathbf{D}_i^{-A_i}=\{D_{ij}: j\notin A_i\}$. Under this design, both $p_{ia_i}$ and $p_{ib_i}$ are unmatched with outcomes, so the query price $p_{i\kappa_i}$ is not measurable from the inputs. This restores the variation in the query price conditioned on the information set required for target identification. We next present our main estimator.

\subsection{The estimator: Decision-Conditioned Masked-Outcome Meta-Learning}

\begin{definition}[Decision-Conditioned Masked-Outcome Meta-Learning (DCMOML)]
\label{def:dcmoml_twopoint}
Fix a hypothesis class $\{g_\Lambda:\Lambda\in\mathcal{H}\}$ with
$g_\Lambda:\ \mathcal{Z}\times\mathbb{R}^K\times\mathbb{R}^{K-2}\to\mathbb{R}^2$, where the inputs
correspond to $(Z_i,p_{i1:K},\mathbf{D}_i^{-A_i})$.
For each task $i$, draw $\kappa_i$ according to Step~1, and
form the masked-outcome information set
\begin{align}
X_i^{-A_i}=\big(Z_i,\ p_{i1:K},\ \mathbf{D}_i^{-A_i}\big).\label{eq:info_drop2_twopoint}
\end{align}
The \emph{DCMOML estimator} is any empirical risk minimizer
\begin{equation}
\widehat{\Lambda}_{\mathrm{DCMOML}}
\in
\arg\min_{\Lambda\in\mathcal{H}}
\frac{1}{N}\sum_{i=1}^N
\Big(
D_{i \kappa_i}-P_{i \kappa_i}^\top g_\Lambda\!\big(X_i^{-A_i}\big)
\Big)^2,
\label{eq:dcmoml_erm_twopoint}
\end{equation}
where $P_{ik}=(1,p_{ik})^\top$. We refer to $g_{\widehat{\Lambda}_{\mathrm{DCMOML}}}$ as the
\emph{DCMOML meta-learner}.
\end{definition}

\begin{remark}
Instead of sampling $\kappa_i\sim\mathrm{Unif}(A_i)$, one can minimize the \emph{average} loss over the two candidate query indices. Both have the same population risk; averaging removes the randomness of the query draw, though the two empirical objectives need not share finite-sample minimizers. We use the averaged form in all experiments. Appendix~\ref{app:two_point} gives a design-weighted regression interpretation of this loss.

\end{remark}

\subsection{Main results: Identification, Consistency, and Maximal Validity}
We now formalize the guarantees of the proposed information design. The first result establishes the identification of the causal target and gives consistency of the empirical procedure, and the second shows that the design is maximally uniformly valid. Proofs are provided in Appendix~\ref{sec:app}.

\begin{theorem}[Identification and consistency of DCMOML]
\label{thm:dcmoml_identification_consistency}
Consider the model in Section~\ref{sec:model}. Suppose that Assumptions~\ref{ass:local_exogeneity} and~\ref{ass:moments} hold, and let
$A_i=\{a_i,b_i\}$ be any locally exogenous pair of indices chosen as a measurable function of $(Z_i,p_{i1:K})$.
Let $\kappa_i\sim\mathrm{Unif}(A_i)$ independently across tasks and independently of all structural variables conditional on $A_i$, and let $X_i^{-A_i}$ denote the masked-outcome information set defined in \eqref{eq:info_drop2_twopoint}. For any measurable $g$ satisfying $\mathbb{E}[
(P_{i\kappa_i}^\top g(X_i^{-A_i}))^2
]<\infty,$ define the population risk
\begin{equation*}
\mathcal{L}(g)
\triangleq
\mathbb{E}\!\left[
\left(
D_{i\kappa_i}
-
P_{i\kappa_i}^\top g(X_i^{-A_i})
\right)^2
\right],
\label{eq:poprisk_dcmoml_twopoint}
\end{equation*}
where the expectation is over the joint law induced by the pricing policy and the two-point design. Define $g^\ast(x)
\triangleq
\mathbb{E}[\Theta_i\mid X_i^{-A_i}=x].$ Then the following hold.

\begin{enumerate}[leftmargin=*]
\item \textbf{Identification.}
The function $g^\ast$ is the unique minimizer of $\mathcal{L}$: if
$\mathcal{L}(g)=\mathcal{L}(g^\ast)$, then
\[
g(X_i^{-A_i})=g^\ast(X_i^{-A_i})
\qquad\text{almost surely}.
\]
Consequently, $\mathbb{E}[\Theta_i\mid X_i^{-A_i}]$ is identified.

\item \textbf{Consistency under realizability.}
Let $\{g_\Lambda:\Lambda\in\mathcal{H}\}$ be the DCMOML hypothesis class and define
$\mathcal{L}(\Lambda)\triangleq\mathcal{L}(g_\Lambda)$. Let
$\widehat{\Lambda}_{\mathrm{DCMOML}}$ be any empirical risk minimizer:
\[
\widehat{\Lambda}_{\mathrm{DCMOML}}
\in
\arg\min_{\Lambda\in\mathcal{H}}
\frac1N\sum_{i=1}^N
\left(
D_{i\kappa_i}
-
P_{i\kappa_i}^\top g_\Lambda(X_i^{-A_i})
\right)^2 .
\]
Assume realizability: there exists $\Lambda_0\in\mathcal{H}$ such that $g_{\Lambda_0}(X_i^{-A_i})= g^\ast(X_i^{-A_i})$ almost surely. Assume further that the induced squared-loss class is Glivenko--Cantelli:
\begin{equation*}
\sup_{\Lambda\in\mathcal{H}}
\left|
\frac1N\sum_{i=1}^N
\left(
D_{i\kappa_i}
-
P_{i\kappa_i}^\top g_\Lambda(X_i^{-A_i})
\right)^2
-
\mathcal{L}(\Lambda)
\right|
\xrightarrow{p}0.
\label{eq:GC_dcmoml_twopoint}
\end{equation*}
Then $\mathcal{L}(g_{\widehat{\Lambda}_{\mathrm{DCMOML}}})
\xrightarrow{p}
\mathcal{L}(g^\ast)$. Moreover, if $X^\circ$ is the masked input of a new task drawn independently of the training data, then $g_{\widehat{\Lambda}_{\mathrm{DCMOML}}}(X^\circ)\xrightarrow{p} g^\ast(X^\circ)$.
\end{enumerate}
\end{theorem}

\begin{remark}[Scope of identification]\label{rem:scope}
DCMOML identifies $\mathbb{E}[\Theta_i\mid X_i^{-A_i}]$ -- the posterior mean of the structural demand parameter given task covariates, realized price history, and unobfuscated outcomes. Four positioning notes follow. First, the target is neither $\Theta_i$ (not identifiable from few within-task observations) nor $g(Z_i)=\mathbb{E}[\Theta_i\mid Z_i]$, which discards price-path information. Second, the target is \emph{policy-conditional}: because prices are endogenous, the estimand shifts if the pricing policy changes. Third, identification requires no excluded instruments or randomized prices; it instead relies on availability of at least two locally exogenous prices at the task-level (Assumption~\ref{ass:local_exogeneity}). Fourth, the target has an explicit observable formula, \eqref{eq:observable_formula} in Appendix~\ref{sec:app}, so identification does not hinge on the squared-loss characterization or on a parameterization.
\end{remark}

The next result shows that the proposed information design is maximally uniformly valid for the squared-loss objective under the two-point query randomization scheme: for every nonconstant refinement that reveals a function of the masked outcomes $(D_{ia_i},D_{ib_i})$ or the query index $\kappa_i$, there is an admissible data-generating process under which the refined population risk fails to recover the refined conditional target uniquely. 
\begin{theorem}[Maximal uniform validity]
\label{prop:local_maximal_validity}
Fix a two-index selection rule \(\mathcal A(z,p_{1:K})\), and suppose there exists
a support point \((\bar z,\bar p_{1:K})\) such that $\mathcal A(\bar z,\bar p_{1:K})=\{a,b\},$ and $\bar p_a\neq \bar p_b$. Given a data-generating process, let $A_i=\{a_i,b_i\}=\mathcal A(Z_i,p_{i1:K})$
and let $X_i^{-A_i}=(Z_i,p_{i1:K},D_i^{-A_i})$ be the masked-outcome information set. Let \(\kappa_i\) be drawn uniformly from
\(A_i\), independently across tasks and independently of
\((Z_i,p_{i1:K},\Theta_i,\epsilon_{i1:K})\) conditional on \(A_i\). Consider any refinement
\[
W_i=(X_i^{-A_i},R_i),
\qquad
R_i=r_i(\kappa_i,D_{ia_i},D_{ib_i}),
\]
where the refinement rule is measurable and its restriction to $\{a,b\}\times\mathbb R^2$ at the support point \((\bar z,\bar p_{1:K})\) is nonconstant. Then there exists a data-generating process satisfying the model assumptions,
including condition~\eqref{ass:zero_mean} for every \(k\in\{1,\dots,K\}\), such that the refined population risk
\[
L_W(g)
=
\mathbb E
\left[
\left(
D_{i\kappa_i}
-
P_{i\kappa_i}^{\top}g(W_i)
\right)^2
\right]
\]
does not identify $\mathbb E[\Theta_i\mid W_i]$.
\end{theorem}

Two failure modes drive this result. If the refinement reveals the query index, the loss constrains only $P_{i\kappa_i}^\top\Theta_i$ at that input and leaves the orthogonal direction free. Otherwise it is a nonconstant function of the masked outcomes, and conditioning on it can select on the query shock, so the conditional target is no longer a minimizer. The theorem is a worst-case statement about the specified population loss over the model class: it does not say that every refined target is statistically unidentified (if the refinement reveals only $\kappa_i$, the target $\mathbb{E}[\Theta_i\mid X_i^{-A_i},\kappa_i]=\mathbb{E}[\Theta_i\mid X_i^{-A_i}]$ remains identified by Theorem~\ref{thm:dcmoml_identification_consistency}; it is the refined loss that cannot select it), nor that richer information is harmful under stronger assumptions.

\section{Experiments}\label{sec:experiments}

\subsection{Evaluating Design Alternatives}\label{sec:eval_design_alts}
We evaluate \textsc{DCMOML} and design alternatives on a more realistic variant of the illustrative example considered in Appendix~\ref{ex:fail}. Each task $i$ is a store-specific linear demand curve with task-varying slope and intercept. Prices are chosen endogenously by a manager who (i) forms a noisy estimate of the task's revenue-optimal price and (ii) experiments locally around that estimate. We vary the noise in this estimate to control confounding: accurate estimates tightly couple prices to latent parameters, while higher noise introduces quasi-exogenous price variation and weakens confounding. We consider high, medium, and low confounding regimes (HC, MC, LC, resp.), always with small within-task experimentation, and $K\in\{2,3,4,5\}$ observations per task. Since prices do not depend on the demand shocks, Assumption~\ref{ass:local_exogeneity} is satisfied for any pair of indices; DCMOML queries the lowest- and highest-priced observations of each task. The full setup appears in Appendix~\ref{app:eval_design_alts_setup}.

{\bf Methods.}
The alternatives isolate key design choices in DCMOML.
\begin{enumerate}[leftmargin=*]
  \item {META.} A classical meta-learning baseline that conditions on $(p_1,D_1),\ldots,(p_{K-1},D_{K-1})$ and minimizes squared error on the final point $(p_K,D_K)$.
  \item {DCML (Decision-conditioned meta-learning).} META that also conditions on the query price $p_K$ (but not $D_K$). This baseline tests whether simply passing $p_K$ to META resolves the confounding problem.
  \item {DCUOML (Decision-conditioned unassigned-outcome meta-learning).} A DCMOML variant that also reveals the non-query pair outcome to the learner, without assigning it to a price.
  \item {EB-GLS (Empirical Bayes with GLS shrinkage).} A hierarchical random-effects baseline that models $(\theta_i^0,\theta_i^1)$ as draws from a shared Gaussian prior estimated from the data, and returns their posterior mean under a Gaussian likelihood with the known noise variance.
  \item {SHARED (Pooled OLS).} A single linear demand model fitted by pooling all tasks and observations.
  \item {TASK-OLS (Per-task OLS).} An independent OLS fit for each task using only its own observations.
\end{enumerate}
We also report a reference, \emph{Oracle}: the Bayes floor $\mathbb E[\Theta_i\mid X_i^{-A_i}]$ for DCMOML's information set, computed by numerical integration under the known data-generating process. It is the target that Theorem~\ref{thm:dcmoml_identification_consistency} identifies.

All four meta-learners (DCMOML, META, DCML, DCUOML) use the same function class, a polynomial map from the method's information set to $(\theta_i^0,\theta_i^1)$ that is symmetric in exchangeable observations (Appendix~\ref{app:eval_design_alts_setup}), fitted exactly by least squares. There is thus no optimization error, and differences between methods reflect only the information design and supervision. For each configuration we train on 20{,}000 tasks, evaluate on 10{,}000 new tasks, and repeat over 10 seeds. Figure~\ref{fig:synthetic} reports MSE for the slope $\theta_i^1$ and intercept $\theta_i^0$; SHARED, DCML and TASK-OLS are far off the plotted scale and appear, with all other results, in Tables~\ref{tab:synth_intercept} and~\ref{tab:synth_slope} (Appendix~\ref{app:eval_design_alts_setup}).

\begin{figure}[t]
  \centering
  \includegraphics[width=\linewidth]{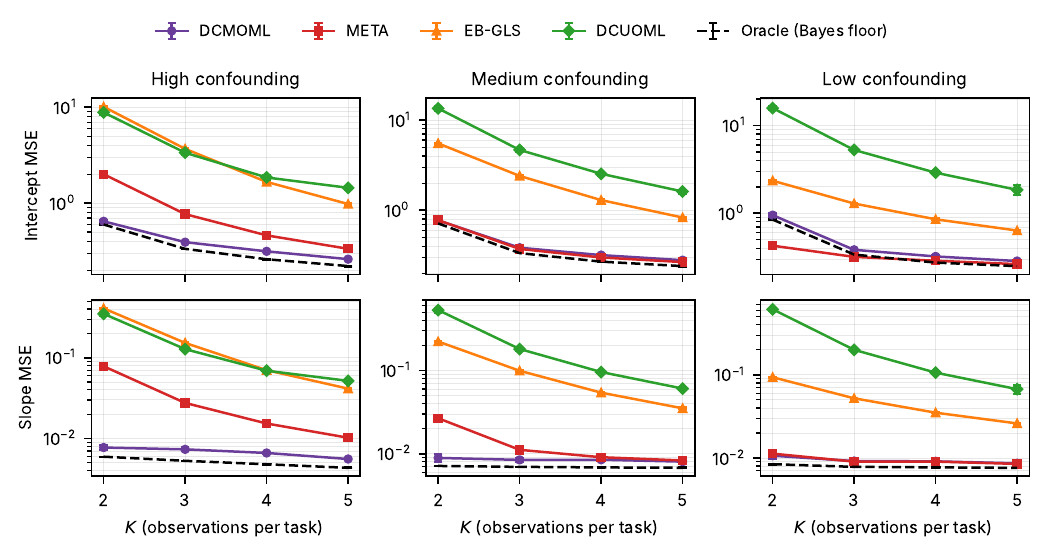}
  \caption{Synthetic study: intercept (top) and slope (bottom) MSE against the number of observations per task, under high, medium and low confounding (log scale; mean over 10 seeds with 95\% intervals, mostly smaller than the markers). The oracle is the Bayes floor for DCMOML's information set.}
  \label{fig:synthetic}
\end{figure}

{\bf Results.} We make four main observations.
(1) Under high confounding, \textsc{DCMOML} is the most accurate method at every $K$, for both parameters and in every seed, and stays close to its Bayes floor: intercept MSE 0.65 against a floor of 0.60 at $K=2$ (\textsc{META} 2.00, \textsc{EB-GLS} 10.1, \textsc{SHARED} 25.8), and 0.26 against 0.22 at $K=5$ (\textsc{META} 0.33). \textsc{DCML} fails at every $K$ (intercept MSE about 96) because once the query price is an input the loss no longer pins down the coefficient vector; with the polynomial learner the failure is literal, as its least-squares design is rank-deficient. \textsc{EB-GLS} improves on per-task estimation through shrinkage, but its prior treats parameters as independent of prices, so it remains well behind \textsc{DCMOML} whenever confounding is present.
(2) As confounding weakens, the gap between \textsc{DCMOML} and \textsc{META} closes. Under medium confounding the two are within about 0.02 in intercept MSE for $K\ge 3$, with \textsc{DCMOML} still clearly better on the slope at $K\le3$. Under low confounding \textsc{META} has the lower intercept MSE at every $K$ (0.43 versus 0.95 at $K=2$; 0.26 versus 0.29 at $K=5$), with similar slope errors: masking discards the demand observations \textsc{META} conditions on, and when prices carry little confounding that information outweighs \textsc{META}'s bias. \textsc{META} can then even beat DCMOML's oracle, which bounds only methods restricted to DCMOML's information set.
(3) \textsc{DCUOML} performs poorly at every $K$ despite query randomization: revealing the non-query outcome lets the learner partly infer the query decision, reintroducing the same failure. This is consistent with Theorem~\ref{prop:local_maximal_validity}, which guarantees such a failure for some process in the model class; the experiment shows it also occurs in this one.
(4) Unmasked outcomes make the problem easier: under high confounding the oracle intercept MSE falls from 0.60 at $K=2$ to 0.22 at $K=5$, and \textsc{DCMOML}'s error falls with it. Theorem~\ref{thm:dcmoml_identification_consistency} implies that the population risk is minimized at the oracle target, so the remaining gap reflects the learner's finite-sample and approximation error. Consistent with this, the gap narrows as tasks are added (Table~\ref{tab:synth_ntasks}): at $K=2$, with 400{,}000 tasks, \textsc{DCMOML} reaches its floor (0.601 versus 0.600) while \textsc{META} stays near 2.0, and at $K=5$ moving from 20{,}000 to 50{,}000 tasks lowers its intercept MSE from 0.26 to 0.24 (floor 0.22).

\subsection{Evaluation on Retail Dataset}\label{sec:retail_eval}
We evaluate {DCMOML} on \emph{UK-online-retail}, an online retail transaction dataset from a UK-based gift retailer spanning 01/12/2010--09/12/2011 \citep{uci_online_retail}. The processed panel records, for each product and day with at least one sale, the transaction price and the number of units sold. It covers 3{,}806 products with an average of 3.79 distinct prices per product (median 4). Since ground-truth demand parameters $\Theta_i$ are unobserved, this experiment cannot test causal recovery; instead, it asks whether the DCMOML information design remains useful for predicting demand at a pricing decision withheld from each product's training data.

{\bf Tasks and holdout protocol.}
We define two product-level tasks. In both, the context is $Z_i\in\mathbb{R}^{1024}$, the sentence-transformer embedding \citep{reimers2019sentencebert} of the product title, and the transfer-learning methods are trained on two observations per product ($K=2$), with a third observation held out.
\begin{enumerate}[leftmargin=*]
\item {\it Static-Top3 (withheld price).}
We retain products with at least three distinct prices ($n=2{,}833$). For each product, we compute average daily demand over the sale days at each of its three most frequently observed prices, producing pairs $(p_{i1},D_{i1}), (p_{i2},D_{i2}), (p_{i3},D_{i3})$ ordered by frequency. We train on the top two pairs and evaluate on the third. The held-out price is never seen for that product during training, and it lies outside the interval spanned by the two training prices for 86\% of products, so this task tests prediction at a less frequently used, mostly out-of-range price.
\item {\it Exposure-Sequence (later pricing episode).}
For each product, we compress the daily series into a sequence of price exposures $\{(p_{ik},e_{ik},D_{ik})\}_{k}$, where $p_{ik}$ is the $k$th distinct price run, $e_{ik}$ is the number of consecutive sale days in that run, and $D_{ik}$ is average daily demand over those days. By construction $p_{i,k}\neq p_{i,k+1}$. We retain products with at least three exposures ($n=3{,}132$), train on the first two, and evaluate on the third. The held-out price may coincide with a training price, but it is observed in a later episode under a different history.
\end{enumerate}

{\bf Methods and protocol.}
We compare (i) {DCMOML}, with inputs $(Z_i,p_{i1},p_{i2})$ for Static-Top3 and $(Z_i,p_{i1},e_{i1},p_{i2},e_{i2})$ for Exposure-Sequence, trained with the loss averaged over query indices $\{1,2\}$; (ii) {META}, which conditions on one support pair and predicts the other, symmetrized over both assignments; (iii) {SHARED}, which predicts $\Theta_i$ from $Z_i$ alone; and (iv) {PER-TASK}, a separate exposure-weighted linear fit per product on all non-holdout prices (Static-Top3 only). Since $D_{ik}$ averages over a variable number of days, all losses weight observations by exposure, normalized within product; all transfer methods use the same two-hidden-layer MLP (Appendix~\ref{app:retail_impl}). For each of 100 seeds, products are split 80/20; observations 1--2 of the 80\% subset are used for fitting and of the 20\% subset for early stopping, and observation 3 is never used for either. We report the held-out RMSE of the predicted demand $\widehat\theta^0_i+\widehat\theta^1_i p_{i3}$ on observation 3 of all eligible products, exposure-weighted; all methods share the split in each seed, so we also report paired differences across seeds.

\begin{table}[t]
\centering
\scriptsize
\caption{UK-online-retail, 100 seeds: mean $\pm$ 95\% CI; paired rows give the mean DCMOML$-$baseline difference and the share of seeds in which DCMOML is lower.}
\label{tab:retail}
\setlength{\tabcolsep}{3.5pt}
\begin{tabular}{llccc}
\toprule
Task & Metric & DCMOML & META & SHARED \\
\midrule
Static-Top3 & Held-out RMSE & \textbf{109.45} $\pm$ 0.06 & 110.08 $\pm$ 0.24 & 110.13 $\pm$ 0.10 \\
 & \quad paired diff.\ (share) & -- & $-0.63\pm0.25$ (68\%) & $-0.68\pm0.09$ (93\%) \\
\midrule
Exposure-Sequence & Held-out RMSE & \textbf{35.08} $\pm$ 0.41 & 39.26 $\pm$ 0.75 & 35.63 $\pm$ 0.25 \\
 & \quad paired diff.\ (share) & -- & $-4.19\pm0.54$ (98\%) & $-0.56\pm0.37$ (80\%) \\
\bottomrule
\end{tabular}
\end{table}

{\bf Findings.}
Table~\ref{tab:retail} summarizes the results. On Static-Top3, DCMOML has the lowest held-out RMSE, below META in 68\% of seeds and below SHARED in 93\%. The mean gap to META is driven partly by a few seeds in which META fits poorly; the median paired difference is $-0.16$. PER-TASK performs poorly (RMSE 200.50), highlighting the gains from cross-task transfer despite product heterogeneity.

On Exposure-Sequence, DCMOML again has the lowest held-out RMSE, below META in 98\% of seeds and below SHARED in 80\%. The fitted slopes, however, are not interpretable as price responses: in the median seed only 39\% of DCMOML's slopes are negative, and this share ranges from 0\% to 100\% across seeds (80\% for META, 67\% for SHARED). Exposures are short and volatile (62\% of training exposures last a single sale day), so the training pairs carry little within-product price information. We therefore read this task as evidence that the information design improves prediction in a later pricing episode, not as evidence about price response.

Two caveats apply to both tasks: neither verifies Assumption~\ref{ass:local_exogeneity} or recovers causal parameters, and transaction prices in this dataset may vary with order size and customer type, with days without sales unrecorded. The retail results are predictive evidence that complements, rather than replicates, the synthetic causal validation of Section~\ref{sec:eval_design_alts}.

\section{Conclusion}\label{sec:limitations}
We show how learning across tasks can improve demand estimates even when historical prices reflect unobserved differences in demand determinants. Our approach uses information contained in those pricing decisions while carefully controlling which outcomes enter the learner's inputs and supervision.

A limitation of our guarantees is that they rely on linear demand and the availability of two distinct locally exogenous prices within each task. Local exogeneity is determined by the price-setting process and cannot be verified from the observed data alone. For example, an adaptive pricing policy can violate this assumption when later prices respond to demand shocks at the selected observations. Even when the local exogeneity assumption holds, closely spaced prices contain less information about the demand slope, and learning a flexible mapping from task histories to demand parameters requires many tasks.

As we note, there are also limits to the benefits of the information design. Masking outcomes sacrifices information, and our experiments show that this cost can outweigh the benefit when confounding is weak. Moreover, the conditional-mean target reflects the historical pricing policy, so a change in that policy may require refitting the learned mapping. The retail experiments demonstrate that the design can improve demand prediction in practice; however, we do not claim that the resulting estimates recover causal price responses.

These limitations suggest several directions for future work: extending the approach to nonlinear demand and cross-price effects, assessing sensitivity to violations of local exogeneity, and investigating practical guidelines for deciding between DCMOML and standard meta-learning approaches.

\begin{ack}
This work was supported by the Discovery Partners Institute (DPI) of the University of Illinois System through a DPI Science Team seed grant (Award No.~11575), ``Industry Data Trusts using Next-Generation AI Technologies.'' We also thank the anonymous referees and the area chair for their helpful comments and feedback. The authors declare no competing interests. 
\end{ack}

\bibliographystyle{plainnat}
\bibliography{refs}

\appendix

\section{Illustrative example}
\label{ex:fail}
\paragraph{Shared-Model Estimation Can Fail under Near-Optimal Local Pricing behavior.}

We illustrate the failure of shared-model learning with a simple example in which learning only the shared component leads to a qualitatively incorrect conclusion. Suppose all stores share the same observable context, $Z_i \equiv Z$ for all  $i$, so that $g(Z_i)$ is constant across tasks. In this case, the estimator in
\eqref{eq:global-naive} reduces to fitting a \emph{single} common parameter vector $\Theta$ using all
pooled observations. Assume that each store $i$ has a true linear demand curve
\[
D_{ik} = \theta_i^0 + \theta_1 p_{ik} + \epsilon_{ik},
\qquad
\theta_1 < 0,
\]
where $\theta_1$ is a common (negative) slope and $\theta_i^0 \ge 0$ is a store-specific intercept.
Suppose further that the store manager observes her local demand curve perfectly and sets prices close to the
revenue-maximizing level
\[
p_i^\ast
=
\arg\max_p \; p(\theta_i^0 + \theta_1 p)
=
-\frac{\theta_i^0}{2\theta_1},
\]
with small idiosyncratic pricing noise: $p_{ik} = p_i^\ast + \varepsilon_{ik}.$
Substituting into the demand equation yields
\begin{align*}
D_{ik} = \theta_i^0 + \theta_1 p_{ik} + \epsilon_{ik} = \frac{\theta_i^0}{2} + \theta_1 \varepsilon_{ik} + \epsilon_{ik} = -\theta_1 p_{ik} + 2\theta_1 \varepsilon_{ik} + \epsilon_{ik}.
\end{align*}
Thus, in the pooled data, demand is positively correlated with price even though the true causal
slope $\theta_1$ is strictly negative. Because stores with higher intercepts optimally charge higher
prices and also exhibit higher demand, the shared-model estimator attributes this cross-store
variation to a positive price effect. As a result, the estimator in \eqref{eq:global-naive} converges to a parameter vector with a
positive price coefficient, reversing the sign of the true causal effect. Figure~\ref{fig:confounding-example}
illustrates this phenomenon: although all stores have downward-sloping demand curves, the pooled
price--demand relationship is upward sloping due to endogenous pricing.

\begin{figure}[ht]
    \centering
    
    \begin{subfigure}[t]{0.45\linewidth}
        \centering
        \includegraphics[width=\linewidth]{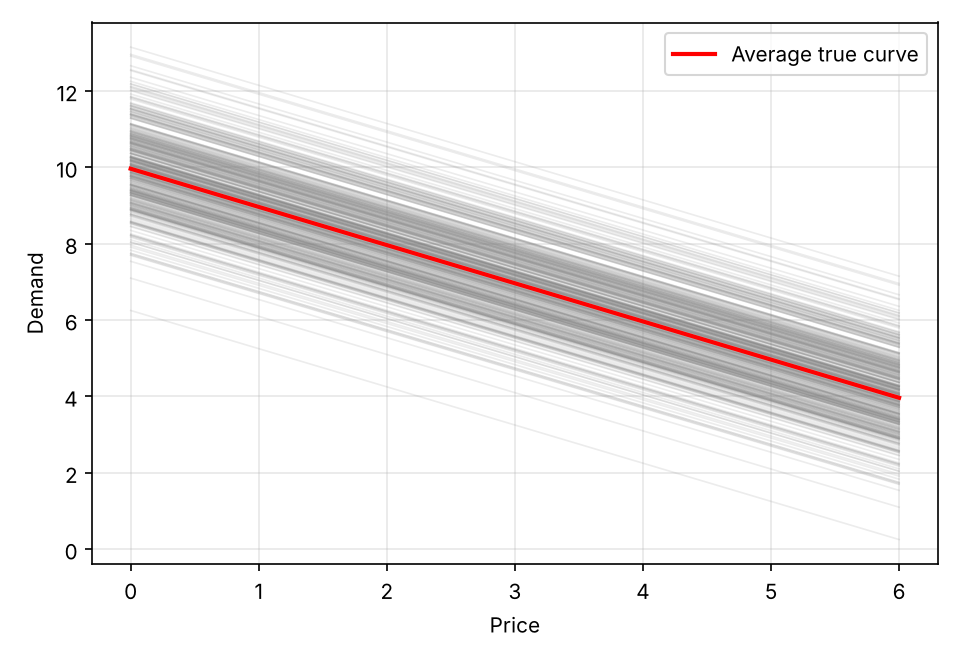}
        \caption{True demand curves across stores.  
        Each store has a linear mean demand curve 
        $D_{ik} = \theta_i^0 + \theta_1 p_{ik}$ 
        with heterogeneous intercepts $\theta_i^0 \sim \mathcal{N}(10, 1)$  
        and identical negative slope 
        $\theta_1 = -1$.}
    \end{subfigure}
    \hfill
    \begin{subfigure}[t]{0.45\linewidth}
        \centering
        \includegraphics[width=\linewidth]{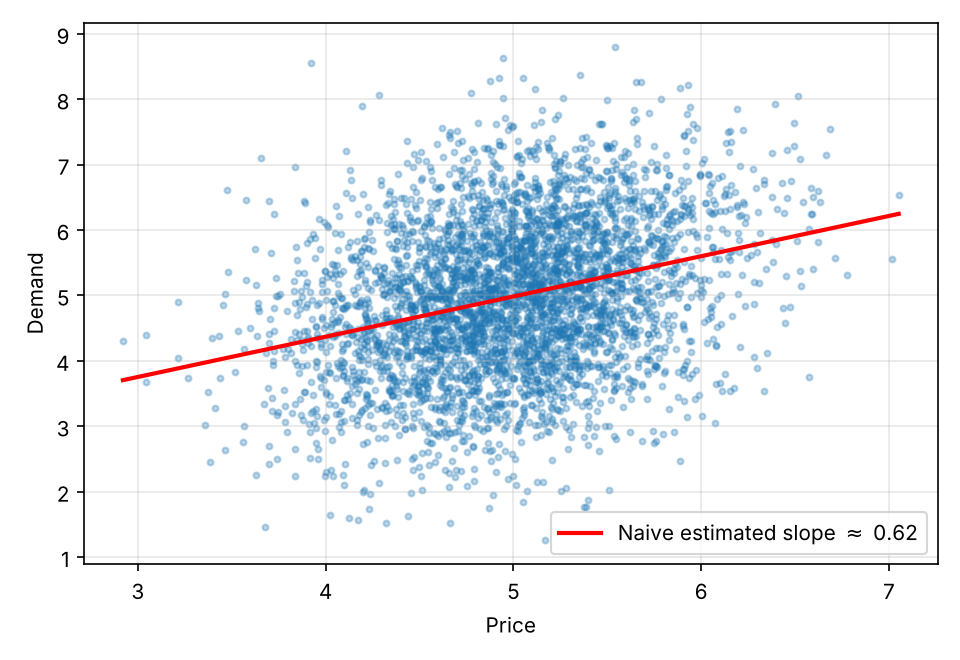}
        \caption{Pooled price--demand data.  
        Managers set prices near the optimum  
        $p_i^\ast = -\theta_i^0/(2\theta_1)$. Demand noise is $\epsilon_{ik} \sim \mathcal{N}(0, 1)$, and pricing noise is $\varepsilon_{ik} \sim \mathcal{N}(0, 0.25^2)$. }
    \end{subfigure}

    \caption{Illustration of confounding in the multi-task pricing setting ($N=2000$, $K=2$).  
    Left: True (causal) demand curves all have negative slopes.  
    Right: Pooled data exhibit a positive price--demand relationship due to optimal pricing 
    responding to store-specific intercepts.}
    \label{fig:confounding-example}
\end{figure}

\paragraph{Failure of Meta-Learning.}
\label{ex:meta_failure}

We continue with the same example as above with $K=2$ observations per store. The
support set consists of a single observation $(p_{i1}, D_{i1})$, and the query observation is
$(p_{i2}, D_{i2})$. Suppose the intercept $\theta_i^0$ is Gaussian, and both the demand noise $\epsilon_{ik}$ and
the pricing noise $\varepsilon_{ik}$ are Gaussian. Under the manager’s pricing rule, the
random vector $\bigl(\Theta_i, p_{i1}, D_{i1}, p_{i2}, D_{i2}\bigr)$ is jointly Gaussian. Consequently, the conditional expectation $\mathbb{E}[\Theta_i \mid p_{i1}, D_{i1}]$
is an affine function of $(p_{i1}, D_{i1})$, and hence, linear adaptation rules of the form 
$\widehat{\Theta}_i = A[
p_{i1}, D_{i1}]^\top+ B$
are sufficient to recover the mapping. Nevertheless, the ERM objective \eqref{eq:meta-erm} remains biased since the
second-period price $p_{i2}$ is endogenously chosen and remains
correlated with $\Theta_i$ even after conditioning on $(p_{i1}, D_{i1})$. Thus, meta-learning fails to recover the causal price-response function. Figure~\ref{fig:pd_theta_distributions}
illustrates this failure.

\begin{figure}[ht]
    \centering
    \begin{subfigure}[t]{0.45\textwidth}
        \centering
        \includegraphics[width=\textwidth]{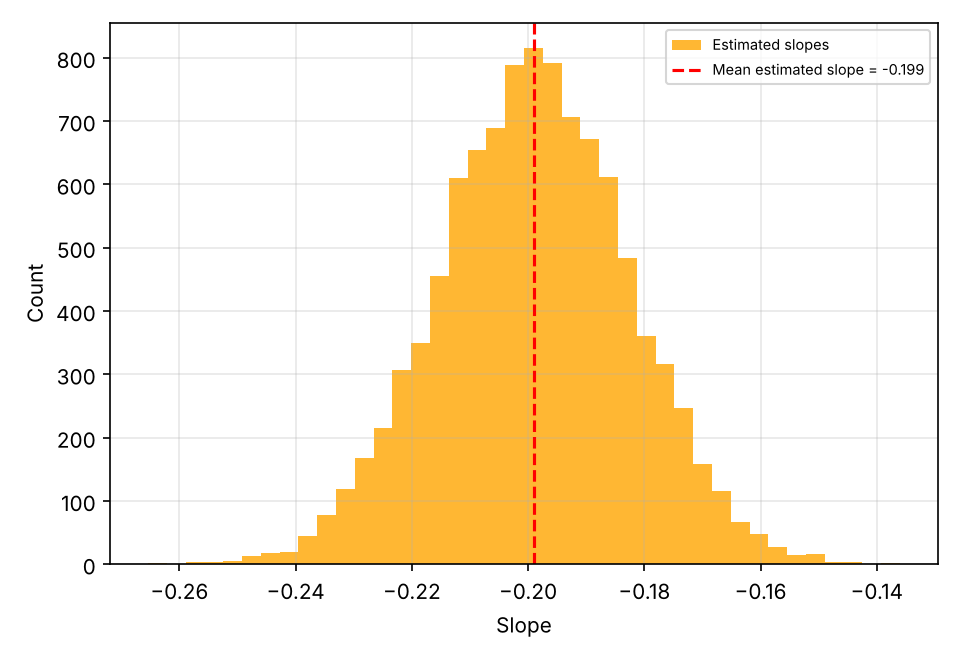}
        \caption{
        Distribution of estimated slopes. The true slope is $-1$ for all tasks.
        }
        \label{fig:pd_slope_dist}
    \end{subfigure}
    \hfill
    \begin{subfigure}[t]{0.45\textwidth}
        \centering
        \includegraphics[width=\textwidth]{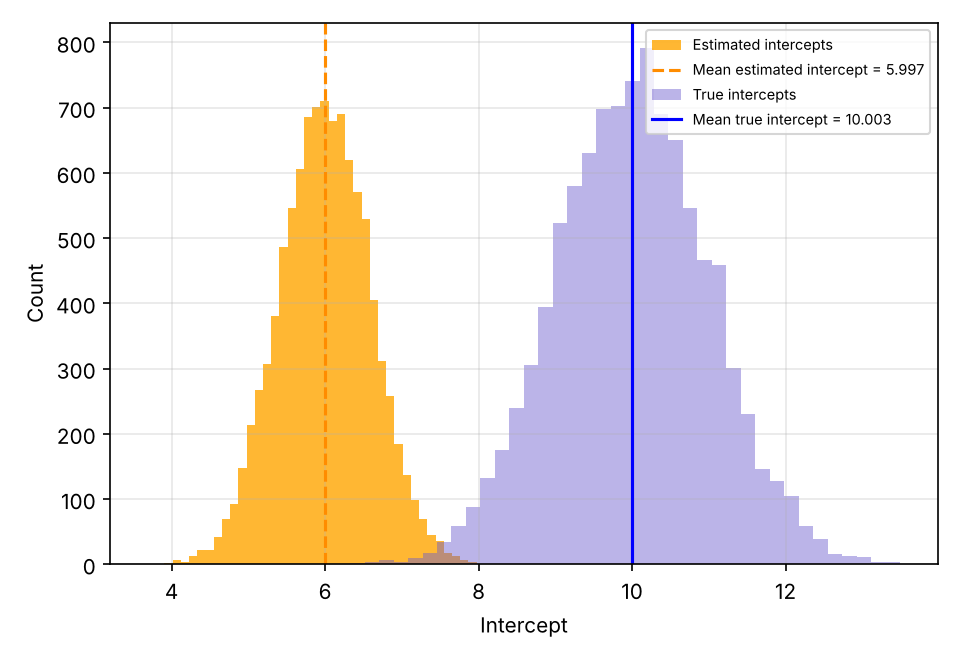}
        \caption{
        Distribution of true and estimated intercepts.
        }
        \label{fig:our_intercept_dist}
    \end{subfigure}

    \caption{
    Estimation performance of the simple outcome-based meta-estimator
    $\hat{\theta}_{ij} = a_j p_{i1} + b_j D_{i1} + c_j$ for $j\in\{0,1\}$ with $K=2$, fitted on 100{,}000 stores and evaluated on 10{,}000 new stores.
    The left panel shows slope estimates while the right panel shows intercept estimates.    }
    \label{fig:pd_theta_distributions}
\end{figure}

\paragraph{Performance of DCMOML.}
We now examine the performance of DCMOML in this example, when each store provides only $K=2$ observations. Since the observable context is constant (so $Z_i\equiv Z$), the two-outcome withholding design in
\eqref{eq:info_drop2_twopoint} leaves the learner with only the realized prices: $X_i^{-A_i}=(p_{i1},p_{i2})$, with $A_i=\{1,2\}$ for every task. Thus, DCMOML
amounts to training a meta-learner that maps the \emph{pair of prices} to the store-level demand
parameters,
\[
g^\ast(p_{i1},p_{i2})=\mathbb{E}[\Theta_i\mid p_{i1},p_{i2}],
\qquad \Theta_i=(\theta_i^0,\theta_i^1)^\top.
\]
As discussed earlier, in this example, all primitives are Gaussian, and the pricing rule is affine, so
$(\Theta_i,p_{i1},p_{i2})$ is jointly Gaussian. Consequently,
$g^\ast(p_{i1},p_{i2})$ is an affine function of $(p_{i1},p_{i2})$, and it is sufficient to fit a
linear predictor. Moreover, $(p_{i1},p_{i2})$ is exchangeable, so the target is symmetric:
$g^\ast(p_{i1},p_{i2})=g^\ast(p_{i2},p_{i1})$. One can therefore impose a symmetric linear form in
which the two prices enter with the same coefficients. Figure~\ref{fig:our_theta_distributions} shows the resulting estimates of the conditional mean demand parameters given the two observed prices.

\begin{figure}[ht]
    \centering
    \begin{subfigure}[t]{0.45\textwidth}
        \centering
        \includegraphics[width=\textwidth]{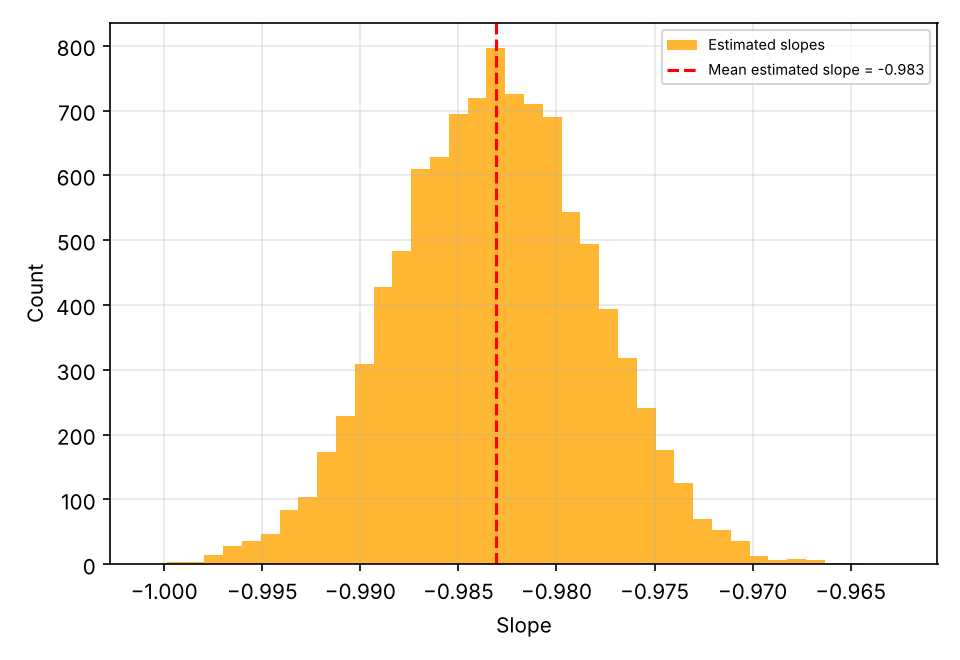}
        \caption{
        Distribution of estimated slopes under DCMOML. The true slope is $-1$ for all tasks.
        }
        \label{fig:our_slope_dist}
    \end{subfigure}
    \hfill
    \begin{subfigure}[t]{0.45\textwidth}
        \centering
        \includegraphics[width=\textwidth]{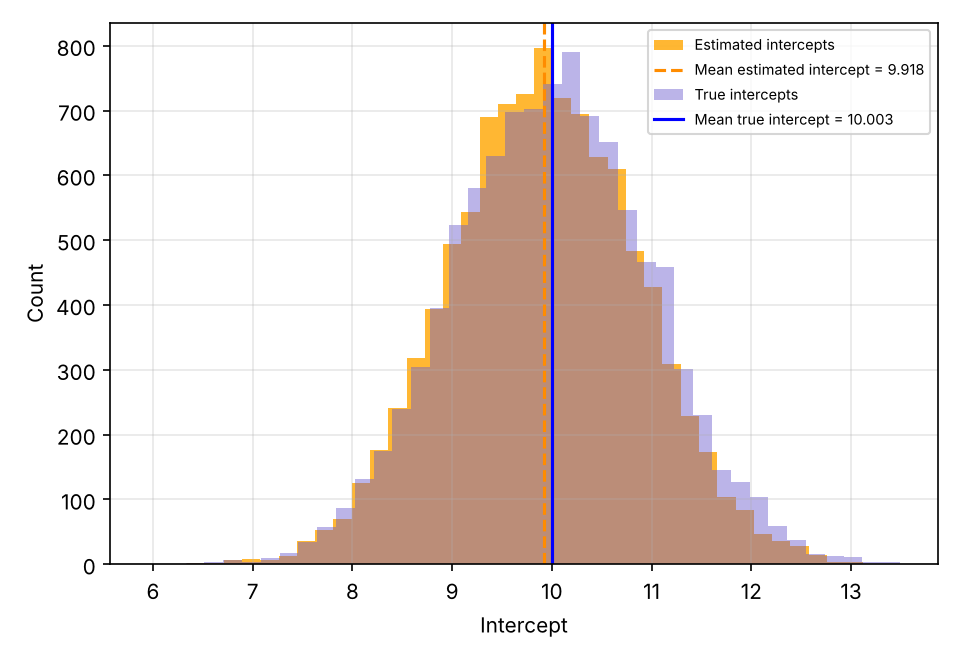}
        \caption{
        Distribution of true and estimated intercepts under DCMOML.        }
        \label{fig:pd_intercept_dist}
    \end{subfigure}

    \caption{
    Estimation performance of the DCMOML meta-learner
    $\hat{\theta}_{ij} = a_j (p_{i1} + p_{i2}) + c_j$ for $j\in\{0,1\}$ with $K=2$, fitted on 100{,}000 stores and evaluated on 10{,}000 new stores.
    }
    \label{fig:our_theta_distributions}
\end{figure}

\section{Proofs of Main Results}\label{sec:app}

{\bf Conventions.} Population expectations involving a fitted predictor are taken over an independent draw from the task population, conditional on the training data $\mathcal T_N$. Fitted estimators and the uniform deviations in Theorem~\ref{thm:dcmoml_identification_consistency} are assumed measurable. Refinement statistics in Theorem~\ref{prop:local_maximal_validity} take values in a standard Borel space, as ordinary real- or vector-valued features do, so that distinct values are measurably distinguishable.

\begin{proof}[Proof of Theorem~\ref{thm:dcmoml_identification_consistency}]
Suppress the task index. Write $X=X^{-A}$ and $A=\{a,b\}$, and abbreviate the query regressor as $P=P_\kappa$. Because $A$ is a measurable function of $(Z,p_{1:K})$, which is part of $X$, the pair $A$ and its prices $p_a,p_b$ are $X$-measurable. By Assumption~\ref{ass:moments}, $\Theta$ and each $P_k^\top\Theta$ are square integrable, and each shock $\epsilon_k$ has finite variance by the model, so each $D_k=P_k^\top\Theta+\epsilon_k$ is square integrable.

{\bf Step 1: conditional regression identity.}
We first show that the two masked shocks have mean zero given the input:
\begin{equation}
\mathbb{E}[\epsilon_a\mid X]=\mathbb{E}[\epsilon_b\mid X]=0.
\label{eq:app_eps_X}
\end{equation}
The labels $a,b$ are random, so we argue separately on each of the finitely many events $E_{k\ell}=\{a=k,b=\ell\}$, which are $(Z,p_{1:K})$-measurable. On $E_{k\ell}$ the input equals $X_{k\ell}=(Z,p_{1:K},\{D_j:j\notin\{k,\ell\}\})$, which is a function of $(Z,p_{1:K},\mathbf D^{-k})$. Hence, for every bounded measurable $h$, Assumption~\ref{ass:local_exogeneity} for index $k$ on $E_{k\ell}$ gives
\[
\mathbb{E}\big[\epsilon_k\mathbf 1_{E_{k\ell}}h(X)\big]
=\mathbb{E}\big[\mathbf 1_{E_{k\ell}}h(X_{k\ell})\,\mathbb{E}[\epsilon_k\mid Z,p_{1:K},\mathbf D^{-k}]\big]=0,
\]
that is, $\mathbb{E}[\epsilon_k\mathbf 1_{E_{k\ell}}\mid X]=0$. The same argument with $\ell$ in place of $k$ gives $\mathbb{E}[\epsilon_\ell\mathbf 1_{E_{k\ell}}\mid X]=0$. Summing over the events gives \eqref{eq:app_eps_X}.

Next, the query. Given $A$, the draw $\kappa$ is uniform on $A$ and independent of $(Z,p_{1:K},\Theta,\epsilon_{1:K})$. Since $A$ and $X$ are functions of these variables, $\kappa$ is also uniform on $A$ and independent of $(\Theta,\epsilon_{1:K})$ given $X$. Therefore
\begin{align*}
\mathbb{E}[\Theta\mid X,\kappa]&=\mathbb{E}[\Theta\mid X]\triangleq g^\ast(X),\\
\mathbb{E}[\epsilon_\kappa\mid X,\kappa]
&=\mathbf 1\{\kappa=a\}\,\mathbb{E}[\epsilon_a\mid X]+\mathbf 1\{\kappa=b\}\,\mathbb{E}[\epsilon_b\mid X]=0,
\end{align*}
the last equality by \eqref{eq:app_eps_X}. Finally, $P=(1,p_\kappa)^\top$ is $(X,\kappa)$-measurable. Taking conditional expectations in $D_\kappa=P^\top\Theta+\epsilon_\kappa$ and pulling $P$ out of the conditional expectation gives
\begin{equation}
\mathbb{E}[D_\kappa\mid X,\kappa]=P^\top g^\ast(X).
\label{eq:app_cond_reg}
\end{equation}
Pulling $P$ out is valid since $\theta^0$ and $\theta^0+p_\kappa\theta^1=P^\top\Theta$ are integrable by Assumption~\ref{ass:moments}, hence so is $p_\kappa\theta^1$.

{\bf Step 2: excess-risk identity and uniqueness.}
Conditional Jensen's inequality applied to \eqref{eq:app_cond_reg} gives $\mathbb{E}[(P^\top g^\ast(X))^2]\le\mathbb{E}[D_\kappa^2]<\infty$, so $g^\ast$ satisfies the integrability requirement of the theorem and has finite risk. Fix any other $g$ that satisfies it, put $\Delta(X)=g(X)-g^\ast(X)$, and let $\xi=D_\kappa-P^\top g^\ast(X)$ be the residual of the target. Both $\xi$ and $P^\top\Delta(X)$ are square integrable, and $\mathbb{E}[\xi\mid X,\kappa]=0$ by \eqref{eq:app_cond_reg}, so
\[
\mathbb{E}\big[\xi\,P^\top\Delta(X)\big]
=\mathbb{E}\big[P^\top\Delta(X)\,\mathbb{E}[\xi\mid X,\kappa]\big]=0 .
\]
Since $D_\kappa-P^\top g(X)=\xi-P^\top\Delta(X)$, expanding the square gives
\begin{equation}
\mathcal L(g)-\mathcal L(g^\ast)
=\mathbb{E}\big[(P^\top\Delta(X))^2\big]
=\mathbb{E}\big[\Delta(X)^\top Q(X)\Delta(X)\big],
\label{eq:app_excess_identity}
\end{equation}
where $Q(X)=\tfrac12 P_aP_a^\top+\tfrac12 P_bP_b^\top$; the second equality conditions on $X$ and averages over the two equally likely values of $\kappa$. For every $v\in\mathbb R^2$, $v^\top Q(X)v=\tfrac12(P_a^\top v)^2+\tfrac12(P_b^\top v)^2$; because $p_a\ne p_b$ almost surely, $P_a$ and $P_b$ are linearly independent and $Q(X)$ is positive definite almost surely. The integrand $\Delta^\top Q\Delta$ is therefore nonnegative and vanishes only where $\Delta(X)=0$, so $\mathcal L(g)=\mathcal L(g^\ast)$ forces $\Delta(X)=0$ almost surely. Thus $g^\ast$ is the unique minimizer up to almost-sure equality. Since $\mathcal L$ depends only on the law of the observed data $(X,D_a,D_b)$ and on the known query randomization, the target is identified.

{\bf Step 3: consistency.}
Let $\widehat{\mathcal L}_N$ be the empirical risk in the theorem and $\delta_N=\sup_{\Lambda\in\mathcal H}|\widehat{\mathcal L}_N(\Lambda)-\mathcal L(\Lambda)|$, so $\delta_N\to0$ in probability by the Glivenko--Cantelli assumption. Let $\widehat\Lambda$ be an empirical risk minimizer and $\Lambda_0$ the parameter in the realizability assumption, so that $\mathcal L(\Lambda_0)=\mathcal L(g^\ast)$. On the event $\{\delta_N<\infty\}$, whose probability tends to one,
\[
\mathcal L(\widehat\Lambda)
\le\widehat{\mathcal L}_N(\widehat\Lambda)+\delta_N
\le\widehat{\mathcal L}_N(\Lambda_0)+\delta_N
\le\mathcal L(g^\ast)+2\delta_N ,
\]
using the definition of $\delta_N$, then the optimality of $\widehat\Lambda$ for the empirical risk, then the definition of $\delta_N$ again. Together with $\mathcal L(\widehat\Lambda)\ge\mathcal L(g^\ast)$ from Step~2, this gives $0\le\mathcal L(\widehat\Lambda)-\mathcal L(g^\ast)\le2\delta_N\to0$ in probability, which is risk consistency.

For coefficient consistency, let $\mathcal T_N$ be the training data (including the query draws), let $X^\circ$ be the masked input of a new task drawn independently of $\mathcal T_N$, and write $\Delta_N(x)=g_{\widehat\Lambda}(x)-g^\ast(x)$ and $r_N=\mathcal L(\widehat\Lambda)-\mathcal L(g^\ast)$. On $\{\delta_N<\infty\}$ the fitted function has finite risk, and hence satisfies the integrability requirement, since $(P^\top g)^2\le2D_\kappa^2+2(D_\kappa-P^\top g)^2$. Applying \eqref{eq:app_excess_identity} to it with $\mathcal T_N$ held fixed gives
\[
\mathbb{E}\big[\Delta_N(X^\circ)^\top Q(X^\circ)\Delta_N(X^\circ)\,\big|\,\mathcal T_N\big]=r_N .
\]
Let $\lambda(x)=\lambda_{\min}(Q(x))$, which is positive almost surely by Step~2 but need not be bounded away from zero. Fix $\varepsilon>0$ and $t>0$. If $\|\Delta_N(X^\circ)\|_2>\varepsilon$, then either $\lambda(X^\circ)\le t$, or $\lambda(X^\circ)>t$ and hence $\Delta_N^\top Q\Delta_N\ge\lambda\|\Delta_N\|_2^2>t\varepsilon^2$ at $X^\circ$. Bounding the probability of the second event by Markov's inequality gives
\[
\mathbb{P}\big(\|\Delta_N(X^\circ)\|_2>\varepsilon\,\big|\,\mathcal T_N\big)
\le\mathbb{P}\big(\lambda(X^\circ)\le t\big)+\min\Big\{1,\frac{r_N}{t\varepsilon^2}\Big\}.
\]
Since $r_N\to0$ in probability, taking expectations over $\mathcal T_N$ and using bounded convergence gives $\limsup_{N}\mathbb{P}(\|\Delta_N(X^\circ)\|_2>\varepsilon)\le\mathbb{P}(\lambda(X^\circ)\le t)$ for every $t>0$. Letting $t\downarrow0$, the right side tends to zero because $\lambda(X^\circ)>0$ almost surely. Hence $g_{\widehat\Lambda}(X^\circ)\to g^\ast(X^\circ)$ in probability.
\end{proof}

\paragraph{Remarks on Theorem~\ref{thm:dcmoml_identification_consistency}.}
Two consequences of the proof are worth recording.

\emph{(i) An observable formula for the target.} Let $m_k(X)=\mathbb{E}[D_k\mid X]$ for $k\in\{a,b\}$. Since $p_k$ is $X$-measurable, \eqref{eq:app_eps_X} gives $m_k=P_k^\top g^\ast$, and solving these two linear equations gives
\begin{equation}
g^\ast(X)=\Big(\frac{p_b\,m_a(X)-p_a\,m_b(X)}{p_b-p_a},\;\frac{m_b(X)-m_a(X)}{p_b-p_a}\Big)^{\!\top}.
\label{eq:observable_formula}
\end{equation}
Every term on the right is determined by the law of the observed data, so this formula establishes identification directly, without reference to the squared-loss risk.

\emph{(ii) Approximate minimization and misspecification.} Exact realizability and exact optimization are not needed for the risk bound. Suppose $\widehat g_N$ satisfies $\widehat{\mathcal L}_N(\widehat g_N)\le\inf_{g\in\mathcal H_N}\widehat{\mathcal L}_N(g)+\eta_N$ over a possibly growing class $\mathcal H_N$, and let $\delta_N$ be the uniform deviation over $\mathcal H_N$. The argument of Step~3, combined with \eqref{eq:app_excess_identity}, gives the oracle inequality
\[
0\le\mathcal L(\widehat g_N)-\mathcal L(g^\ast)
\le\inf_{g\in\mathcal H_N}\mathbb{E}\big[(g-g^\ast)^\top Q\,(g-g^\ast)\big]+2\delta_N+\eta_N ,
\]
so risk consistency follows whenever the approximation term, $\delta_N$ and $\eta_N$ all vanish.

\begin{proof}[Proof of Theorem~\ref{prop:local_maximal_validity}]
Suppress the task index. Fix the support point $(\bar z,\bar p_{1:K})$ from the statement, and write $r(k,u,v)$ for the refinement rule at that point, a measurable function on $\{a,b\}\times\mathbb R^2$ that is nonconstant by assumption. Write $P_j=(1,\bar p_j)^\top$; since $\bar p_a\ne\bar p_b$, the vectors $P_a$ and $P_b$ are linearly independent.

In each of the two cases below we construct a process with the following common features: $Z\equiv\bar z$ and $p_{1:K}\equiv\bar p_{1:K}$, so that $A=\{a,b\}$ always; $\Theta\equiv\bar\theta$ is deterministic (the shared component equals $\bar\theta$ at $\bar z$ and $\Omega\equiv0$); the shocks $\epsilon_1,\dots,\epsilon_K$ are i.i.d., bounded and mean zero; $\kappa$ is uniform on $\{a,b\}$ and independent of everything else; and tasks are i.i.d. Because prices and coefficients are deterministic, $\mathbf D^{-j}$ is a function of the shocks other than $\epsilon_j$, which are independent of $\epsilon_j$, so $\mathbb{E}[\epsilon_j\mid Z,p_{1:K},\mathbf D^{-j}]=\mathbb{E}[\epsilon_j]=0$ for every $j=1,\dots,K$. Thus every model assumption holds, including condition~\eqref{ass:zero_mean} at every index and Assumption~\ref{ass:moments}. Because $\Theta$ is deterministic, the refined target is $\mathbb{E}[\Theta\mid W]=\bar\theta$. It remains to show that the refined risk either has a minimizer other than $\bar\theta$ or is not minimized by $\bar\theta$.

{\bf Case 1: the refinement reveals the query index.}
Suppose $r(a,u,v)\ne r(b,u,v)$ for some $(u,v)\in\mathbb R^2$. Let $\bar\theta$ be the unique vector with $P_a^\top\bar\theta=u$ and $P_b^\top\bar\theta=v$, and set all shocks to zero. Then $(D_a,D_b)=(u,v)$ always, so $R=r(\kappa,u,v)$ takes different values for $\kappa=a$ and $\kappa=b$, and $\kappa$ is a measurable function of $R$, hence of $W$. Consider
\[
\widetilde g(W)=\bar\theta+(\bar p_\kappa,-1)^\top .
\]
Because $P_\kappa^\top(\bar p_\kappa,-1)^\top=0$, we have $P_\kappa^\top\widetilde g(W)=P_\kappa^\top\bar\theta=D_\kappa$. Both $\bar\theta$ and $\widetilde g$ therefore have zero refined risk, so both minimize $L_W$, and $\widetilde g\ne\bar\theta$ everywhere. The refined risk has non-unique minimizers.

{\bf Case 2: the refinement does not reveal the query index.}
Otherwise $r(a,s,t)=r(b,s,t)\triangleq F(s,t)$ for all $(s,t)$, so $R=F(D_a,D_b)$ and $F$ is nonconstant. A function on $\mathbb R^2$ that is constant along every horizontal and every vertical line is constant, so $F$ changes along one of its two coordinates. Relabeling $a$ and $b$ if necessary, there exist $u_0<u_1$ and $c$ with $F(u_0,c)\ne F(u_1,c)$. Put $\mu=(u_0+u_1)/2$ and $d=(u_1-u_0)/2>0$, and let $\bar\theta$ be the unique vector with $P_a^\top\bar\theta=\mu$ and $P_b^\top\bar\theta=c$. Let the shocks be i.i.d.\ with
\[
\mathbb{P}(\epsilon_j=-d)=\mathbb{P}(\epsilon_j=d)=\tfrac18,\qquad \mathbb{P}(\epsilon_j=0)=\tfrac34 .
\]
Then $D_a=\mu+\epsilon_a$ and $D_b=c+\epsilon_b$. Let $B=\{R=F(u_0,c)\}$, an event determined by $W$, and let $m=\mathbb{E}[\epsilon_a\mathbf 1_B]$. We bound $m$ by splitting on the values of $(\epsilon_a,\epsilon_b)$:
\begin{itemize}[leftmargin=*,itemsep=0pt]
\item if $\epsilon_a=0$, the contribution is zero;
\item if $\epsilon_a=-d$ and $\epsilon_b=0$, then $(D_a,D_b)=(u_0,c)$, so $B$ occurs; this contributes $-d\cdot\tfrac18\cdot\tfrac34$;
\item if $\epsilon_a=d$ and $\epsilon_b=0$, then $(D_a,D_b)=(u_1,c)$, so $B$ does not occur; this contributes zero;
\item if $\epsilon_a=-d$ and $\epsilon_b\ne0$, the contribution is at most zero;
\item if $\epsilon_a=d$ and $\epsilon_b\ne0$, the contribution is at most $d\cdot\tfrac18\cdot\tfrac14$.
\end{itemize}
Hence $m\le-\tfrac{3d}{32}+\tfrac{d}{32}=-\tfrac{d}{16}<0$, and in particular $\mathbb{P}(B)>0$. Let $v=(-\bar p_b,1)^\top/(\bar p_a-\bar p_b)$, which satisfies $P_a^\top v=1$ and $P_b^\top v=0$, and for $t>0$ consider the predictor
\[
g_t(W)=\bar\theta-t\,v\,\mathbf 1_B .
\]
Since $D_\kappa-P_\kappa^\top g_t(W)=\epsilon_\kappa+t\,\mathbf 1\{\kappa=a\}\mathbf 1_B$, expanding the square gives
\[
L_W(g_t)-L_W(\bar\theta)
=2t\,\mathbb{E}\big[\epsilon_\kappa\mathbf 1\{\kappa=a\}\mathbf 1_B\big]+t^2\,\mathbb{E}\big[\mathbf 1\{\kappa=a\}\mathbf 1_B\big]
=t\,m+\tfrac{t^2}{2}\,\mathbb{P}(B),
\]
where the last step uses that $R=F(D_a,D_b)$ does not involve $\kappa$, so $\kappa$ is independent of $(\epsilon_a,\mathbf 1_B)$ and $\mathbb{P}(\kappa=a)=\tfrac12$. Taking $t=d/16$ and using $\mathbb{P}(B)\le1$,
\[
L_W(g_t)-L_W(\bar\theta)\le-\frac{d^2}{256}+\frac{d^2}{512}=-\frac{d^2}{512}<0 .
\]
All risks are finite because prices, coefficients and shocks take finitely many values. The refined target $\bar\theta=\mathbb{E}[\Theta\mid W]$ is therefore not a minimizer of the refined risk.

The two cases are exhaustive. In Case~1 the refined risk has distinct minimizers, and in Case~2 it is not minimized by the refined target; in either case it does not identify $\mathbb{E}[\Theta\mid W]$.
\end{proof}

\section{Experimental setup details for Section~\ref{sec:eval_design_alts}}
\label{app:eval_design_alts_setup}

{\bf Data-generating process.} Each task $i$ follows the linear demand model
\[
D_{ik} \;=\; \theta_i^0 + \theta_i^1\,p_{ik} + \epsilon_{ik}, \qquad k=1,\ldots,K,
\]
with $\theta_i^0 \sim \mathcal{N}(10,1)$ and $\theta_i^1 \sim \mathcal{N}(-1,0.1^2)$ truncated to $\theta_i^1\le -0.3$ (the truncation keeps optimal prices positive and binds with negligible probability). Demand shocks are i.i.d.\ $\epsilon_{ik}\sim\mathcal{N}(0,0.5^2)$, independent of the task parameters and of the manager's pricing noise; the standard deviation is about one tenth of mean demand.
Prices are set by a manager who targets a noisy version of the revenue-optimal price $p_i^\ast=-\theta_i^0/(2\theta_i^1)$:
\[
\tilde p_i \;=\; p_i^\ast + \zeta_i,\quad \zeta_i \sim \mathcal{N}\big(0,(\sigma_c|p_i^\ast|)^2\big),
\qquad
p_{ik} \;=\; \tilde p_i + \nu_{ik},\quad \nu_{ik}\sim\mathcal{N}\big(0,(0.1|\tilde p_i|)^2\big),
\]
where $\sigma_c \in \{0,0.1,0.2\}$ controls confounding strength (HC, MC, LC). Prices depend on $\Theta_i$ but not on the demand shocks, so Assumption~\ref{ass:local_exogeneity} holds for every pair of indices. There are no covariates ($Z_i$ is constant).

{\bf Query pair and information sets.} DCMOML uses the pair $A_i=\{\arg\min_k p_{ik},\arg\max_k p_{ik}\}$, selected from prices only; for $K=2$ this is $\{1,2\}$. Its information set is $X_i^{-A_i}=(p_{i1:K},\mathbf D_i^{-A_i})$. META conditions on observations $1,\dots,K-1$ and is supervised by observation $K$. DCML adds $p_{iK}$ to META's inputs. DCUOML adds the non-query pair outcome to DCMOML's inputs without its price index; each training task contributes both query assignments, and predictions average the two.

{\bf Learner.} All meta-learners (DCMOML, META, DCML, DCUOML) use the same function class, a symmetric polynomial $g(x)=(\phi(x)^\top b_0,\ \phi(x)^\top b_1)$. Observations that play the same role within a task are exchangeable in this DGP, so the learner sees them as a set, through their means and (co)variances: DCMOML's unmasked observations, and META's (and DCML's) support set. Let $f$ collect the first-order inputs: the pair prices for DCMOML, the query price for DCML, the unassigned outcome for DCUOML, and the means of the set. Let $s$ collect the set's variances of prices and demands and their covariance (present when the set has at least two observations). The basis $\phi$ consists of all monomials of the standardized $f$ up to degree 4, plus each standardized element of $s$ multiplied by all monomials of $f$ up to degree 2. For DCMOML this gives 15 terms per parameter at $K=2$, 70 at $K=3$ and 115 at $K\in\{4,5\}$. Coefficients minimize the method's empirical squared loss, $\sum_r (y_r-b_0^\top\phi(x_r)-p_r\,b_1^\top\phi(x_r))^2$, solved exactly by least squares (minimum-norm solution when the design is rank-deficient, as it is for DCML). DCMOML minimizes the loss averaged over both query indices in $A_i$. Using an exactly solvable learner removes optimization error from the comparison, so differences between methods reflect information design and supervision.

{\bf Other baselines.} EB-GLS models $\Theta_i\sim\mathcal N(\mu,\Sigma)$ independently of prices, with a Gaussian likelihood whose standard deviation is the known noise level $0.5$; $(\mu,\Sigma)$ are estimated by EM on the training tasks, and the estimate is the posterior mean. SHARED is pooled OLS. TASK-OLS is OLS on each task's own $K$ observations. With $K=2$, its slope error is $(\epsilon_{ia}-\epsilon_{ib})/(p_{ia}-p_{ib})$, and the price gap has positive density at zero, so its mean squared error is infinite. Sample averages are therefore unstable, and we report the median squared error.

{\bf Oracle.} The oracle is the Bayes floor for DCMOML's information set, $\mathbb E[\Theta_i\mid X_i^{-A_i}]$, computed by numerical integration under the known DGP. The posterior is evaluated on a $201\times201$ grid over $(\theta^0,\theta^1)$ (six prior standard deviations). The price likelihood integrates out the manager's target on a 1{,}200-point grid, and the likelihood of the unmasked demands is included. Refining the grids changes the oracle MSEs by less than $10^{-4}$. The oracle uses no learner, so it independently checks how close DCMOML comes to its identified target.

{\bf Protocol.} For each $K\in\{2,3,4,5\}$ and $\sigma_c\in\{0,0.1,0.2\}$, and for each of 10 seeds, we draw 20{,}000 training tasks and 10{,}000 new test tasks, fit every method on the training tasks, and compute MSE for $\theta_i^0$ and $\theta_i^1$ on the test tasks. We report the mean across seeds $\pm\,1.96$ standard errors. Tables~\ref{tab:synth_intercept} and~\ref{tab:synth_slope} give all results. Table~\ref{tab:synth_ntasks} varies the number of training tasks under high confounding for $K=2$ and $K=5$, comparing DCMOML and META. Code to reproduce every number is in the supplementary material.

\begin{table}[ht]
\centering
\scriptsize
\setlength{\tabcolsep}{4pt}
\caption{Synthetic study: intercept MSE ($\theta_i^0$). Mean over 10 seeds $\pm\,1.96$ SE; each seed trains on 20{,}000 tasks and evaluates on 10{,}000 new tasks. $^\dagger$Median squared error, because per-task OLS has no finite mean squared error (see text).}
\label{tab:synth_intercept}
\begin{tabular}{llccc}
\toprule
$K$ & Method & $\sigma_c=0$ (HC) & $\sigma_c=0.1$ (MC) & $\sigma_c=0.2$ (LC) \\
\midrule
2 & DCMOML & 0.647 $\pm$ 0.015 & 0.775 $\pm$ 0.028 & 0.953 $\pm$ 0.055 \\
 & META & 2.004 $\pm$ 0.039 & 0.779 $\pm$ 0.034 & 0.428 $\pm$ 0.011 \\
 & DCML & 96.4 $\pm$ 0.4 & 96.3 $\pm$ 0.4 & 94.4 $\pm$ 0.6 \\
 & DCUOML & 8.846 $\pm$ 0.440 & 13.5 $\pm$ 0.8 & 15.8 $\pm$ 0.9 \\
 & EB-GLS & 10.1 $\pm$ 0.2 & 5.514 $\pm$ 0.112 & 2.351 $\pm$ 0.045 \\
 & SHARED & 25.8 $\pm$ 0.2 & 15.1 $\pm$ 0.1 & 5.728 $\pm$ 0.102 \\
 & TASK-OLS$^\dagger$ & 24.5 & 24.9 & 25.1 \\
 & Oracle & 0.599 $\pm$ 0.005 & 0.716 $\pm$ 0.006 & 0.849 $\pm$ 0.005 \\
\midrule
3 & DCMOML & 0.392 $\pm$ 0.011 & 0.385 $\pm$ 0.014 & 0.384 $\pm$ 0.010 \\
 & META & 0.769 $\pm$ 0.047 & 0.374 $\pm$ 0.018 & 0.319 $\pm$ 0.010 \\
 & DCML & 96.7 $\pm$ 0.5 & 96.2 $\pm$ 0.4 & 94.3 $\pm$ 0.7 \\
 & DCUOML & 3.357 $\pm$ 0.191 & 4.664 $\pm$ 0.205 & 5.271 $\pm$ 0.188 \\
 & EB-GLS & 3.685 $\pm$ 0.076 & 2.403 $\pm$ 0.047 & 1.293 $\pm$ 0.019 \\
 & SHARED & 25.8 $\pm$ 0.2 & 15.1 $\pm$ 0.1 & 5.635 $\pm$ 0.042 \\
 & TASK-OLS$^\dagger$ & 8.33 & 8.44 & 8.29 \\
 & Oracle & 0.334 $\pm$ 0.002 & 0.336 $\pm$ 0.003 & 0.339 $\pm$ 0.003 \\
\midrule
4 & DCMOML & 0.315 $\pm$ 0.004 & 0.319 $\pm$ 0.005 & 0.323 $\pm$ 0.008 \\
 & META & 0.462 $\pm$ 0.023 & 0.300 $\pm$ 0.013 & 0.288 $\pm$ 0.007 \\
 & DCML & 96.6 $\pm$ 0.3 & 96.5 $\pm$ 0.7 & 94.5 $\pm$ 0.6 \\
 & DCUOML & 1.854 $\pm$ 0.102 & 2.539 $\pm$ 0.120 & 2.904 $\pm$ 0.110 \\
 & EB-GLS & 1.665 $\pm$ 0.034 & 1.298 $\pm$ 0.032 & 0.855 $\pm$ 0.017 \\
 & SHARED & 25.8 $\pm$ 0.1 & 15.1 $\pm$ 0.1 & 5.715 $\pm$ 0.089 \\
 & TASK-OLS$^\dagger$ & 4.88 & 4.89 & 4.9 \\
 & Oracle & 0.260 $\pm$ 0.002 & 0.270 $\pm$ 0.002 & 0.275 $\pm$ 0.002 \\
\midrule
5 & DCMOML & 0.261 $\pm$ 0.005 & 0.281 $\pm$ 0.007 & 0.286 $\pm$ 0.004 \\
 & META & 0.334 $\pm$ 0.025 & 0.267 $\pm$ 0.007 & 0.264 $\pm$ 0.004 \\
 & DCML & 96.7 $\pm$ 0.4 & 96.1 $\pm$ 0.3 & 93.7 $\pm$ 0.7 \\
 & DCUOML & 1.447 $\pm$ 0.074 & 1.616 $\pm$ 0.079 & 1.851 $\pm$ 0.242 \\
 & EB-GLS & 0.980 $\pm$ 0.016 & 0.833 $\pm$ 0.020 & 0.637 $\pm$ 0.014 \\
 & SHARED & 25.8 $\pm$ 0.2 & 15.0 $\pm$ 0.2 & 5.750 $\pm$ 0.061 \\
 & TASK-OLS$^\dagger$ & 3.48 & 3.43 & 3.5 \\
 & Oracle & 0.221 $\pm$ 0.002 & 0.241 $\pm$ 0.002 & 0.250 $\pm$ 0.002 \\
\bottomrule
\end{tabular}
\end{table}

\begin{table}[ht]
\centering
\scriptsize
\setlength{\tabcolsep}{4pt}
\caption{Synthetic study: slope MSE ($\theta_i^1$). Mean over 10 seeds $\pm\,1.96$ SE; each seed trains on 20{,}000 tasks and evaluates on 10{,}000 new tasks. $^\dagger$Median squared error, because per-task OLS has no finite mean squared error (see text).}
\label{tab:synth_slope}
\begin{tabular}{llccc}
\toprule
$K$ & Method & $\sigma_c=0$ (HC) & $\sigma_c=0.1$ (MC) & $\sigma_c=0.2$ (LC) \\
\midrule
2 & DCMOML & 0.0077 $\pm$ 0.0005 & 0.0089 $\pm$ 0.0010 & 0.0107 $\pm$ 0.0010 \\
 & META & 0.0783 $\pm$ 0.0020 & 0.0267 $\pm$ 0.0015 & 0.0113 $\pm$ 0.0003 \\
 & DCML & 3.962 $\pm$ 0.017 & 4.066 $\pm$ 0.016 & 4.389 $\pm$ 0.019 \\
 & DCUOML & 0.349 $\pm$ 0.019 & 0.531 $\pm$ 0.034 & 0.614 $\pm$ 0.025 \\
 & EB-GLS & 0.408 $\pm$ 0.007 & 0.223 $\pm$ 0.004 & 0.0936 $\pm$ 0.0016 \\
 & SHARED & 1.001 $\pm$ 0.008 & 0.574 $\pm$ 0.006 & 0.203 $\pm$ 0.004 \\
 & TASK-OLS$^\dagger$ & 0.98 & 1 & 1.05 \\
 & Oracle & 0.0059 $\pm$ 0.0000 & 0.0071 $\pm$ 0.0001 & 0.0084 $\pm$ 0.0001 \\
\midrule
3 & DCMOML & 0.0073 $\pm$ 0.0004 & 0.0084 $\pm$ 0.0004 & 0.0092 $\pm$ 0.0001 \\
 & META & 0.0275 $\pm$ 0.0018 & 0.0111 $\pm$ 0.0006 & 0.0090 $\pm$ 0.0003 \\
 & DCML & 3.975 $\pm$ 0.019 & 4.081 $\pm$ 0.014 & 4.391 $\pm$ 0.024 \\
 & DCUOML & 0.128 $\pm$ 0.008 & 0.181 $\pm$ 0.007 & 0.200 $\pm$ 0.007 \\
 & EB-GLS & 0.152 $\pm$ 0.003 & 0.0991 $\pm$ 0.0020 & 0.0522 $\pm$ 0.0007 \\
 & SHARED & 1.002 $\pm$ 0.006 & 0.577 $\pm$ 0.005 & 0.201 $\pm$ 0.003 \\
 & TASK-OLS$^\dagger$ & 0.33 & 0.334 & 0.339 \\
 & Oracle & 0.0053 $\pm$ 0.0000 & 0.0069 $\pm$ 0.0001 & 0.0078 $\pm$ 0.0001 \\
\midrule
4 & DCMOML & 0.0066 $\pm$ 0.0001 & 0.0085 $\pm$ 0.0001 & 0.0091 $\pm$ 0.0003 \\
 & META & 0.0153 $\pm$ 0.0011 & 0.0091 $\pm$ 0.0006 & 0.0091 $\pm$ 0.0004 \\
 & DCML & 3.982 $\pm$ 0.013 & 4.082 $\pm$ 0.023 & 4.413 $\pm$ 0.019 \\
 & DCUOML & 0.0690 $\pm$ 0.0036 & 0.0958 $\pm$ 0.0045 & 0.106 $\pm$ 0.004 \\
 & EB-GLS & 0.0700 $\pm$ 0.0015 & 0.0541 $\pm$ 0.0012 & 0.0351 $\pm$ 0.0007 \\
 & SHARED & 1.001 $\pm$ 0.005 & 0.576 $\pm$ 0.006 & 0.204 $\pm$ 0.003 \\
 & TASK-OLS$^\dagger$ & 0.194 & 0.195 & 0.2 \\
 & Oracle & 0.0048 $\pm$ 0.0000 & 0.0068 $\pm$ 0.0001 & 0.0077 $\pm$ 0.0000 \\
\midrule
5 & DCMOML & 0.0056 $\pm$ 0.0001 & 0.0080 $\pm$ 0.0002 & 0.0086 $\pm$ 0.0002 \\
 & META & 0.0102 $\pm$ 0.0010 & 0.0083 $\pm$ 0.0002 & 0.0085 $\pm$ 0.0002 \\
 & DCML & 3.970 $\pm$ 0.017 & 4.070 $\pm$ 0.012 & 4.399 $\pm$ 0.023 \\
 & DCUOML & 0.0517 $\pm$ 0.0031 & 0.0606 $\pm$ 0.0028 & 0.0672 $\pm$ 0.0078 \\
 & EB-GLS & 0.0414 $\pm$ 0.0008 & 0.0351 $\pm$ 0.0008 & 0.0260 $\pm$ 0.0006 \\
 & SHARED & 1.003 $\pm$ 0.007 & 0.576 $\pm$ 0.006 & 0.205 $\pm$ 0.002 \\
 & TASK-OLS$^\dagger$ & 0.137 & 0.136 & 0.142 \\
 & Oracle & 0.0043 $\pm$ 0.0000 & 0.0068 $\pm$ 0.0001 & 0.0076 $\pm$ 0.0001 \\
\bottomrule
\end{tabular}
\end{table}

\begin{table}[ht]
\centering
\scriptsize
\setlength{\tabcolsep}{4pt}
\caption{Synthetic study, high confounding: MSE as the number of training tasks grows (mean over 5 seeds for $K=2$ and 3 seeds for $K=5$). The oracle does not depend on the number of training tasks.}
\label{tab:synth_ntasks}
\begin{tabular}{rrcccccc}
\toprule
 & & \multicolumn{3}{c}{Intercept MSE} & \multicolumn{3}{c}{Slope MSE} \\ \cmidrule(lr){3-5}\cmidrule(lr){6-8}
$K$ & Tasks & DCMOML & META & Oracle & DCMOML & META & Oracle \\
\midrule
2 & 2{,}000 & 0.948 & 2.107 & 0.597 & 0.0203 & 0.0837 & 0.0059 \\
 & 5{,}000 & 0.917 & 2.163 & 0.596 & 0.0155 & 0.0863 & 0.0060 \\
 & 20{,}000 & 0.637 & 2.020 & 0.601 & 0.0074 & 0.0789 & 0.0059 \\
 & 100{,}000 & 0.607 & 2.014 & 0.599 & 0.0063 & 0.0795 & 0.0059 \\
 & 400{,}000 & 0.601 & 1.993 & 0.600 & 0.0061 & 0.0787 & 0.0060 \\
\midrule
5 & 5{,}000 & 0.397 & 0.403 & 0.218 & 0.0098 & 0.0131 & 0.0043 \\
 & 20{,}000 & 0.259 & 0.353 & 0.223 & 0.0056 & 0.0110 & 0.0044 \\
 & 50{,}000 & 0.240 & 0.297 & 0.220 & 0.0049 & 0.0089 & 0.0043 \\
\bottomrule
\end{tabular}
\end{table}

\clearpage
\section{Implementation Details for Retail Experiments}
\label{app:retail_impl}

This appendix documents the data processing, objective, architecture and optimization used in Section~\ref{sec:retail_eval}. The released code reproduces every number in Section~\ref{sec:retail_eval} and in this appendix.

\subsection{Data processing}
\label{app:retail_data}
The processed UCI Online Retail panel has one row per product, transaction price and day with at least one sale, giving the number of units sold. Days without sales are not recorded, so a product's ``exposure'' at a price counts sale days, and average daily demand is an average over sale days. Each product's 1024-dimensional title embedding, computed from the product description with the sentence-transformers model \texttt{all-roberta-large-v1} \citep{reimers2019sentencebert,allrobertalarge}, is attached as $Z_i$. The UCI Online Retail dataset is distributed under the Creative Commons Attribution 4.0 International (CC BY 4.0) license \citep{uci_online_retail}, and the embedding model under the Apache 2.0 license.

{\it Static-Top3.} For each product and price we count the number of sale days at that price and compute average daily demand over those days. Prices are ranked by this count within product; products with at least three distinct prices are kept ($n=2{,}833$). Ranks 1--2 form the training pairs and rank 3 is held out. Among these products, the top two prices account on average for 89.7\% of the sale days at the top three prices.

{\it Exposure-Sequence.} Rows with non-positive prices are removed. Daily sales are summed per product and day, and the day's price is the median transaction price. Consecutive sale days with the same price form a run, giving $(p_{ik}, e_{ik}, D_{ik})$ = (run price, number of sale days, average daily sales). Products with at least three runs are kept ($n=3{,}132$); the first two runs are used for training and the third is held out.

\subsection{Exposure-weighted supervision}
\label{app:retail_loss_impl}
Each outcome $D_{ik}$ averages over $e_{ik}$ days, so longer exposures are less noisy. Let $\mathcal{Q}_i=\{1,2\}$ be the supervised query indices of product $i$. With per-product normalized weights $\tilde w_{ik}=e_{ik}/\sum_{j\in\mathcal{Q}_i}e_{ij}$, the per-product loss is
\[
\mathcal{L}_i \;=\; \sum_{k\in\mathcal{Q}_i}\tilde w_{ik}\,\big(\hat D_{ik}-D_{ik}\big)^2,
\qquad
\hat D_{ik}=\hat\theta^0_{i}+\hat\theta^1_{i}p_{ik},
\]
where $(\hat\theta^0_i,\hat\theta^1_i)$ comes from the method's own information set (for META, from the support point opposite to query $k$). The training objective is the minibatch average of $\mathcal{L}_i$, and the same loss on validation products is used for early stopping. At evaluation, META's two support-specific estimates are averaged before predicting observation 3.

\subsection{Architecture and optimization}
\label{app:retail_arch_impl}
All methods predict $(\theta^0_i,\theta^1_i)$ with the same MLP,
\[
\texttt{Linear}(d_{\mathrm{in}},256)\rightarrow\texttt{ReLU}\rightarrow
\texttt{Linear}(256,256)\rightarrow\texttt{ReLU}\rightarrow
\texttt{Linear}(256,2),
\]
where $d_{\mathrm{in}}$ is the method-specific input dimension (1024 for SHARED; 1026 for DCMOML and META on Static-Top3; 1028 for DCMOML and 1027 for META on Exposure-Sequence). Inputs are used without normalization. We use AdamW with weight decay $10^{-5}$, batch size 512, learning rate $10^{-3}$ (Static-Top3) or $3\times10^{-4}$ (Exposure-Sequence), at most 100 epochs, and early stopping with patience 10 on the validation loss, restoring the best checkpoint. Seeds $1,\dots,100$ set the product split and the initialization; all methods share the split within a seed. Confidence intervals are mean $\pm\,1.96\times$ (standard deviation across seeds)$/\sqrt{100}$.

{\it Training instability.} On Static-Top3, META's fitted slopes are mostly positive in 10 of 100 seeds (never for DCMOML or SHARED). Section~\ref{sec:retail_eval} reports the median paired difference, which is robust to these seeds. On Exposure-Sequence all three methods have seeds with near-zero average predicted demand (META 17, SHARED 13, DCMOML 3): seeds in which the exposure-weighted average predicted demand at the held-out price is below 10\% of the exposure-weighted average training demand.

{\it PER-TASK.} Weighted least squares per product on all non-holdout prices (weights $e_{ik}$), solved in closed form; it has no randomness.

\section{The DCMOML risk as design-weighted regression of two-point estimates}
\label{app:two_point}
Suppress the task index and use the notation of Appendix~\ref{sec:app}: $X=X^{-A}$, $A=\{a,b\}$ and $Q(X)=\tfrac12 P_aP_a^\top+\tfrac12 P_bP_b^\top$. Let $H$ be the $2\times2$ matrix with rows $P_a^\top$ and $P_b^\top$, which is invertible because $p_a\ne p_b$, and let
\[
\widetilde\Theta=H^{-1}\begin{pmatrix}D_a\\ D_b\end{pmatrix}
\]
be the two-point least-squares estimate of $\Theta$, the coefficients of the line through the two selected observations. Both $H$ and $Q$ are functions of $X$.

\begin{prop}\label{prop:two_point}
For every measurable $g$, pointwise,
\begin{equation}
\frac12\sum_{k\in A}\big(D_k-P_k^\top g(X)\big)^2=\big(\widetilde\Theta-g(X)\big)^\top Q(X)\big(\widetilde\Theta-g(X)\big).
\label{eq:two_point_identity}
\end{equation}
Consequently, the DCMOML population risk is $\mathcal L(g)=\mathbb{E}\big[(\widetilde\Theta-g(X))^\top Q(X)(\widetilde\Theta-g(X))\big]$. Moreover, under the assumptions of Theorem~\ref{thm:dcmoml_identification_consistency},
\[
H^{-1}\begin{pmatrix}\mathbb{E}[D_a\mid X]\\ \mathbb{E}[D_b\mid X]\end{pmatrix}=g^\ast(X).
\]
If $\widetilde\Theta$ is integrable, this also gives $\mathbb{E}[\widetilde\Theta\mid X]=g^\ast(X)$. The weighted-loss identity does not require this additional integrability.
\end{prop}

\begin{proof}
By construction $P_k^\top\widetilde\Theta=D_k$ for $k\in A$, so $D_k-P_k^\top g(X)=P_k^\top(\widetilde\Theta-g(X))$, and
\[
\frac12\sum_{k\in A}\big(P_k^\top(\widetilde\Theta-g)\big)^2=(\widetilde\Theta-g)^\top\Big(\tfrac12\sum_{k\in A}P_kP_k^\top\Big)(\widetilde\Theta-g)=(\widetilde\Theta-g)^\top Q(\widetilde\Theta-g).
\]
Since $\kappa$ is uniform on $A$ given $X$ and independent of the outcomes given $X$, $\mathcal L(g)=\mathbb{E}\big[\tfrac12\sum_{k\in A}(D_k-P_k^\top g(X))^2\big]$, which gives the expression for $\mathcal L$. Finally, $\mathbb{E}[D_k\mid X]=P_k^\top g^\ast(X)$ for $k\in A$ by \eqref{eq:app_eps_X} and the demand model. Multiplying these two conditional means by $H^{-1}$ gives $H^{-1}H\,g^\ast(X)=g^\ast(X)$. If $\widetilde\Theta$ is integrable, taking its conditional expectation yields the same expression because $H$ is $X$-measurable.
\end{proof}

DCMOML is therefore a regression of each task's two-point estimate $\widetilde\Theta$ on its masked input $X$, weighted by the design matrix $Q(X)$: the fitted map shrinks the noisy task-level estimate toward what the cross-task information in $X$ predicts. Writing $\Delta=\widetilde\Theta-g(X)$, $\bar p=(p_a+p_b)/2$ and $h=(p_b-p_a)/2$, the weighted loss is
\[
\Delta^\top Q\Delta=(\Delta_0+\bar p\,\Delta_1)^2+h^2\Delta_1^2,
\]
the squared error at the midpoint price plus the slope error scaled by the half-gap. The slope of $\widetilde\Theta$ is least reliable when the two prices are close, and this is exactly where it receives the least weight. The weighting is essential. Given $X$, the slope of $\widetilde\Theta$ has error $(\epsilon_b-\epsilon_a)/(p_b-p_a)$; when the price gap has positive density at zero, as with Gaussian prices, $\widetilde\Theta$ has infinite variance and need not even be integrable, so an unweighted regression of $\widetilde\Theta$ on $X$ has infinite risk. By \eqref{eq:two_point_identity}, the weighted risk is finite for every $g$ satisfying the integrability requirement of Theorem~\ref{thm:dcmoml_identification_consistency}.

\end{document}